%% file: main.tex
\documentclass{article} 
\usepackage{iclr2027_conference,times}

\iclrfinalcopy
\input{math_commands.tex}

\usepackage{amsthm}

\usepackage{hyperref}
\usepackage[capitalise]{cleveref}
\usepackage{url}

\definecolor{BSUb}{RGB}{0,50,160}     
\definecolor{BSUb2}{RGB}{0,30,96}     
\definecolor{BSUb3}{RGB}{0,114,206}   
\definecolor{BSUo}{RGB}{250,70,22}    
\definecolor{BSUo2}{RGB}{255,130,0}   

\newtheorem{theorem}{Theorem}

\newtheorem{example}{Example}
\newtheorem{remark}{Remark}

\usepackage{amsmath}
\usepackage{amssymb}
\usepackage{nicefrac}       
\usepackage[dvipsnames]{xcolor}
\usepackage{colortbl}
\usepackage{subcaption}
\usepackage{graphicx}
\usepackage{algorithm}
\usepackage{algpseudocode}

\usepackage{booktabs}       
\usepackage{multirow}
\usepackage{tikz}
\usetikzlibrary{arrows.meta,positioning,matrix,calc}
\usepackage{wrapfig}

\definecolor{bestgray}{gray}{0.85}
\newcommand{\first}[1]{\textcolor{YellowOrange}{\mathbf{#1}}}
\newcommand{\second}[1]{\textcolor{CadetBlue}{\mathbf{#1}}}
\newcommand{\third}[1]{\textcolor{RawSienna}{\mathbf{#1}}}
\usepackage{placeins}

\newcommand{\val}[2]{$#1$\,{\small$\pm#2$}}
\newcommand{\bval}[2]{$\mathbf{#1}$\,{\small$\bm{\pm#2}$}}
\newcommand{\uval}[2]{\underline{$#1$\,{\small$\pm#2$}}}

\title{
\centering
Path Laplacian Encodings\\[-0.1em]
for Directed Graphs\\[0.5em]
}

\author{
\parbox{0.96\textwidth}{\centering\normalfont
\textbf{Lydia Mezrag}$^{1,2}$
\qquad
\textbf{Semih Cantürk}$^{1,2}$
\qquad
\textbf{Michael Perlmutter}$^{3}$
\\[0.25em]
\textbf{Bastian Rieck}$^{4}$
\qquad
\textbf{Guy Wolf}$^{1,2}$
\\[0.8em]
\small
$^{1}$Université de Montréal
\qquad
$^{2}$Mila -- Quebec AI Institute
\\[0.25em]
$^{3}$Boise State University
\qquad
$^{4}$Université de Fribourg
\\[0.5em]
\footnotesize
\texttt{lydia.mezrag@umontreal.ca}
\qquad
\texttt{semih.canturk@mila.quebec}
\qquad
\texttt{mperlmutter@boisestate.edu}
\\
\texttt{bastian.grossenbacher@unifr.ch}
\qquad
\texttt{wolfguy@mila.quebec}
}
}

\begin{document}

\maketitle
\lhead{Preprint}

\begin{abstract}
  Directed graphs naturally model many real-world systems in which interactions are \emph{asymmetric}, such as citation networks, web graphs, and information-flow networks. However, graph learning methods commonly rely on message passing with symmetrized graph representations or positional encodings that only partially exploit edge directionality. We introduce \emph{PathLapPE}, a novel spectral positional encoding (PE) derived from the path Laplacian on directed graphs. PathLapPE provides node- and edge-level features that encode directional higher-order structure and can be incorporated into standard graph learning architectures. Empirical results on node- and graph-level benchmark tasks show that PathLapPE yields consistent improvements across several architectures, especially when combined with direction-aware message passing.
  Compared with magnetic Laplacian positional encodings, a widely studied spectral positional encoding for directed graphs, PathLapPE does not require additional fine-tuning of directionality hyperparameters while offering competitive runtime and performance.
\end{abstract}

\section{Introduction}
Many real-world relational systems are naturally represented by directed graphs: citations point from one document to another, hyperlinks encode navigation, circuits and programs contain logical dependencies, and biological or transportation networks often involve intrinsically asymmetric interactions. In such settings, edge orientation is not a secondary attribute but an integral part of the semantic content of the data. Nevertheless, a large fraction of graph neural network architectures either assume undirected graphs or preprocess directed graphs by replacing each directed edge with an undirected one. This naive symmetrization can destroy key information for downstream tasks.

Several approaches have been proposed to incorporate directionality into graph learning architectures. 
These include modifying message passing so that information flowing along incoming and outgoing edges is processed differently \citep[Dir-GNN,][]{rossi2024edge}, utilizing directed Laplacians \citep[MagNet,][]{zhang2021magnet}, and injecting directed information through positional or structural encodings (PSE), attention biases, or edge features \citep{geisler2023transformers,huang2025good}.  

Positional and structural encodings have become central tools in modern graph learning by enriching node and edge representations~\citep{dwivedi2022graph,canturk2024graph,huang2024stability}. Positional encodings typically aim to distinguish where nodes lie in the graph, for instance through Laplacian eigenvectors, random-walk features, distances, or resistance-based quantities. Structural encodings, in contrast, describe the role or local environment of a node or edge, such as degrees, orbit counts, local subgraph statistics, or other invariants. 

Despite substantial progress on positional encodings for undirected graphs, direction-aware positional encodings remain comparatively underdeveloped. Existing directed encodings include directed walk encodings~\citep{geisler2023transformers}, walk profiles~\citep{huang2025good}, and magnetic Laplacian encodings \citep{geisler2023transformers,huang2025good}. These methods have demonstrated that direction-aware encodings can significantly improve the expressivity of graph Transformers and GNNs on tasks where orientation matters. However, magnetic-Laplacian-based encodings also come with practical limitations. Their behavior is sensitive to the choice of the magnetic phase parameter $q$, and strong performance often requires either careful tuning of this parameter or the concatenation of encodings associated with several values of $q$. While multi-$q$ constructions improve the ability of magnetic encodings to represent directed walk patterns, they also require several complex-valued spectral decompositions and increase the size of the resulting positional encoding. This creates a trade-off between directional expressivity and computational simplicity.


In this work, we propose \emph{PathLapPE}, a positional encoding for directed
graphs constructed from the \emph{path Laplacians} of the directed path
complex. Instead of encoding direction through complex phases, PathLapPE uses
the algebraic structure of directed paths to produce real-valued node- and edge-level spectral features. Our contributions are as follows:

\begin{enumerate}
    \item We introduce \emph{PathLapPE}, a spectral PE derived from the path complex of a directed graph. It provides node- and edge-level features, with the latter encoding both edge orientation and directed 2-path structure. Unlike magnetic Laplacian PEs, it requires no phase parameter tuning or multi-\(q\) stacking. 
    It can be easily incorporated into standard GNN and graph Transformer architectures. To the best of our knowledge, this is the first direction-aware positional encoding for graph learning based on the path Laplacian.

    \item We introduce \emph{EdgeSPE}, an extension of SPE~\citep{huang2024stability} to edge-level encodings, which allows for performance gains with edge-level PathLapPE features.  

    \item We show that PathLapPE captures directionality and consistently matches or outperforms other spectral PEs on an array of node- and graph-level directed benchmarks across several GNN architectures, with the largest gains for direction-aware models such as bidirectional and edge-gated GNNs.

\end{enumerate}

\section{Related Work}

\paragraph{Bidirectional graph neural networks.}
Classical message passing neural networks (MPNNs) update node representations by aggregating features from local neighborhoods. Architectures such as GCN~\citep{kipf2017semi}, GAT~\citep{velickovic2018graph}, and GIN~\citep{xu2019how} differ in their choices of aggregation and update functions and usually assume an undirected graph or operate on a symmetrized version of the input. This can be problematic for directed graphs, where incoming and outgoing edges may encode different types of relationships.

Recent work has therefore proposed message passing schemes that explicitly preserve edge orientation. Bidirectional (or directed) GNN adapts standard message passing architectures to directed graphs by aggregating messages from in-neighbors and out-neighbors separately~\citep{jaume2019edgnnsimplepowerfulgnn,huang2025good,rossi2024edge}. Given a graph $G=(V,\vec E)$ and node features $\{\mathbf h_v\}_{v \in V}$, the features for the next layer $\{\mathbf h_v\}_{v \in V}$ after message passing can be written as

\begin{equation*}
    \mathbf h_v' =
    \alpha\, \psi_{\mathrm{in}}\Big(\mathbf h_v,
    \operatorname*{AGG}_{u \in \mathcal{N}_{\mathrm{in}}(v)} \phi_{\mathrm{in}}(\mathbf h_u, \mathbf h_v)\Big)
    + (1-\alpha)\, \psi_{\mathrm{out}}\Big(\mathbf h_v,
    \operatorname*{AGG}_{u \in \mathcal{N}_{\mathrm{out}}(v)} \phi_{\mathrm{out}}(\mathbf h_u, \mathbf h_v)\Big),
\end{equation*}

where the $\alpha$ parameter controls the weight assigned to each direction. Other directed or edge-aware variants, including GAT- and GIN-type architectures with edge features, can also exploit orientation-dependent information when it is made available through the input representation. Our work is complementary to these architectural approaches: rather than modifying the message passing scheme, we construct direction-aware spectral features that encode higher-order relations that these message passing schemes can better leverage.

\paragraph{Spectral and directed positional encodings.}
Positional encodings enrich node or edge representations with information about their location in the graph. Spectral positional encodings are commonly obtained from eigenvectors of a graph operator, most often the Laplacian. Given a Laplacian \(\mathbf L\), a \(k\)-dimensional Laplacian positional encoding (LapPE) assigns to each node \(v\) the vector $\operatorname{PE}(v) = \left(\mathbf u_1(v),\ldots,\mathbf u_k(v)\right)$,
where $\mathbf u_1,\ldots, \mathbf u_k$ are the first $k$ eigenvectors of $\mathbf L$ (i.e., those corresponding to the smallest eigenvalues).
These encodings capture global geometric information that may not be available to purely local message passing layers. However, Laplacian eigenvectors come with sign and basis ambiguities: each eigenvector is defined only up to sign, and eigenspaces associated with repeated eigenvalues are defined only up to a change of orthonormal basis. This has motivated invariant and equivariant approaches to spectral encodings, such as SignNet, BasisNet~\citep{lim2023sign}, and SPE~\citep{huang2024stability}.

For directed graphs, positional encodings must also account for edge orientation. Existing approaches include directed walk profiles and magnetic-Laplacian-based encodings. The magnetic Laplacian is a complex Hermitian operator whose phases encode edge direction, and it has been used both in directed spectral neural architectures such as MagNet and as a source of direction-aware positional encodings~\citep{zhang2021magnet,geisler2023transformers}. Recent multi-\(q\) magnetic Laplacian (Multi-$q$-MagLapPE) encodings concatenate spectral features obtained from several choices of the magnetic phase parameter \(q\), improving the ability to represent directed walk profiles~\citep{huang2025good}. This additional expressivity, however, comes with the need to tune, compute, and store several complex-valued spectral decompositions and increases the size of the resulting encoding.

Our method, PathLapPE, is similar in spirit to these magnetic approaches but differs in its underlying Laplacian. Instead of encoding orientation through complex phases, it uses the path Laplacians associated with the directed path complex~\citep{grigor2022advances}. This yields real-valued encodings that can be used at both the node and edge levels.

\paragraph{Path complexes and directed motifs in graph learning.}
Higher-order directed structure has entered graph learning in two main ways.
MotifNet~\citep{monti2018motifnet} defines convolutions from symmetric node-level Laplacians of a predefined set of directed motifs, without lifting the graph. A second line of work lifts the graph to a path complex~\citep{truong2024weisfeiler}\footnote{The path complex of
\citet{truong2024weisfeiler} is based on simple paths in simple graphs, with a
path identified with its reverse, and differs from the directed path complex
used here.}, and runs message passing on that domain, in analogy with simplicial
and cellular architectures. However, such higher-order message-passing schemes are computationally expensive due to the processing of higher-order tensors~\citep{ppgn}, and the theoretical expressivity guarantees provided by higher-order MPNNs typically fail to account for the practical limitations of GNNs, leading only to modest performance gains at a significant computational cost~\citep{kemper2025what}.


Our approach differs from both. PathLapPE does not require choosing motifs, as directed higher-order structures arise naturally from the boundary maps of the path complex. In contrast to lifted architectures, we use the path Laplacian only to compute node- and edge-level positional encodings; the downstream model then runs standard message passing on the original graph, so it receives directional and higher-order information from the path complex without the cost of higher-order message passing.

\section{From directed graphs to positional encodings}\label{sec:methods}

\subsection{Path Laplacian}
Path Laplacians are a family of Laplace-type operators for directed graphs, derived from the boundary maps of the path complex. We briefly recall here the construction of path complexes and path Laplacians associated with directed graphs, and refer the reader to \citet{grigor2022advances} and \citet{grigor2014homotopy} for more details.   

We consider a directed graph \(G=(V,\vec{E})\) 
and assume throughout that $G$ has no self-loops. An \emph{elementary \(p\)-path} on \(V\) is an ordered sequence of vertices \(v_0 v_1 \ldots v_p\), where repetitions are allowed (and we do not require that $(v_i,v_{i+1})\in \vec{E}$). We call a \emph{\(p\)-path} any finite formal sum
$
\sigma = \sum_{v_0,\ldots,v_p \in V} \lambda_{v_0\ldots v_p} \, v_0\ldots v_p,
$ 
with coefficients \(\lambda_{v_0\ldots v_p} \in \mathbb{R}\). We define the vector space $\Lambda_p$ as the space of all $p$-paths where elements of the form $v_iv_i$ are identified with the zero vector. 
An elementary \(p\)-path \(v_0v_1 \ldots v_p\) is called \emph{allowed} if \((v_i,v_{i+1})\) is a directed edge in \(\vec E\) for all \(0 \leq i \leq p-1\). 
We denote by \(\mathcal{A}_p\) the vector space spanned by the allowed \(p\)-paths. For any \(p>0\), we define the boundary map $\partial_p: \Lambda_p \rightarrow \Lambda_{p-1}$ by first defining
\begin{equation}\label{eq:bdry_Lambda}
    \partial_p(v_0\ldots v_p)
    =
    \sum_{k=0}^{p} (-1)^k
    v_0\ldots v_{k-1}v_{k+1}\ldots v_p,
\end{equation}
for elementary $p$-paths and then extending $\partial_p$
to all of $\Lambda_p$ by linearity. We set $\Lambda_{-1}= \mathcal{A}_{-1}= \{0\}$ to be the zero vector space and let $\partial_{0} : \Lambda_0 \rightarrow \Lambda_{-1}$ be the zero map. These maps satisfy $\partial_{p-1}\circ \partial_p = 0$,
and therefore define a chain complex (see \citet[Lemma 2.1]{grigor2020path})
\[
\cdots \xrightarrow{\partial_{p+1}} \Lambda_p
\xrightarrow{\partial_p} \Lambda_{p-1}
\xrightarrow{\partial_{p-1}} \cdots.
\]

\begin{example}\label{ex: directed triangle under stand allowed paths}
Consider the directed triangle graph shown  in Figure~\ref{fig:directed-motifs}(a). The graph consists of three vertices or 0-paths $a,b$, and $c$. Any ordered sequence of the vertices defines an elementary path, for instance, the ordered pair $ca$ is an elementary $1$-path. However, $ca$ is not an allowed path because $(c,a)\notin \vec{E}$. This graph has exactly three  allowed $1$-paths corresponding to the three edges in $\vec{E}$, $u:=ab, v:=bc$, and $w:=ac$. The only allowed $2$-path is given by $t := abc$, which can be obtained by concatenating the directed edges $(a,b)$ and $(b,c)$. To understand the definition of $\mathcal{A}_1$, we observe that $u$ and $v$ are allowed 1-paths and so the (non-elementary) $1$-path defined by $u+2v$ is an element of $\mathcal{A}_1$, the vector space spanned by allowed $1$-paths, even though it is not itself an allowed $1$-path.      
\end{example}

The following example illustrates that, in general, the boundary of an allowed path is not necessarily an  allowed path and  that the spaces \(\mathcal{A}_p\) do not necessarily form a chain complex.
\begin{example}
    Consider the directed path graph with three vertices, $V=\{a,b,c\}$ and two edges, $\vec{E}=\{(a,b),(b,c)\}$ (shown as a subgraph in Figure~\ref{fig:directed-motifs}(a)). Then, the $2$-path $abc$ is an allowed path and, by definition, $\partial_2(abc) = bc - ac + ab$. However, $ac$ is not an allowed $1$-path and there is no way to write $\partial_2(abc)$ as a linear combination of the allowed paths. Thus, $\partial_2(abc) \notin \mathcal{A}_1$.
\end{example}

In order to obtain a chain complex, one considers the subspace \(\Omega_p\) of \emph{\(\partial\)-invariant allowed paths}, 
\begin{equation}\label{eq:Omega}
    \Omega_p
    =
    \left\{
    \sigma \in \mathcal{A}_p
    :
    \partial_p(\sigma) \in \mathcal{A}_{p-1}
    \right\}. 
\end{equation}
In this paper, we will mostly be interested in the vector space $\Omega_p$ and so, unless otherwise stated, the boundary maps will always refer to the linear maps $\partial_p:\Omega_p\to \Omega_{p-1}$ obtained by restricting the boundary maps in \Eqref{eq:bdry_Lambda} (denoted by the same symbol in a minor abuse of notation). We equip each \(\Lambda_p\) with the standard inner product for which the elementary $p$-paths form an orthonormal basis, and we equip \(\Omega_p \subseteq \Lambda_p\) with the induced inner product.  
The adjoint of the boundary map \(\partial_p:\Omega_p\to \Omega_{p-1}\) is denoted by \(\partial_p^*:\Omega_{p-1}\to \Omega_p\). The \(p\)-path Laplacian is defined by $\mathbf{L}_p
=
\partial_{p+1}\partial_{p+1}^*
+
\partial_p^*\partial_p$. Equivalently, if \(\mathbf{B}_p\) denotes the matrix representation of \(\partial_p\), then, in an orthonormal basis, the $p$-path Laplacian has the matrix representation (also denoted $\mathbf{L}_p$ in a minor abuse of notation)
\begin{equation}\label{eq:PathLap}
    \mathbf{L}_p = \mathbf{B}_{p+1}\mathbf{B}_{p+1}^{\top}
    +
    \mathbf{B}_{p}^{\top}\mathbf{B}_{p}.
\end{equation}
The first term $\mathbf{B}_{p+1}\mathbf{B}_{p+1}^{\top}$ is called the \emph{up Laplacian}, while the second term $  \mathbf{B}_{p}^{\top}\mathbf{B}_{p}$ is called the \emph{down Laplacian}.
Furthermore, unlike magnetic-Laplacian constructions, the path Laplacian is obtained directly from the directed path structure of the graph and does not require the choice of an additional phase parameter. Moreover, because the operator acts on spaces of directed paths, it naturally produces spectral information at different orders. A fully worked example is provided in \cref{ex:example3} (Appendix~\ref{app:proofs}). 

\subsection{Path Laplacian PEs}\label{subsec:PathLapPEs}
In this work, for computational efficiency, we construct path complexes up to path length $p=2$. 
The corresponding path Laplacians provide spectral information at both the vertex and edge levels: the \(0\)-path Laplacian \(\mathbf{L}_0\) acts on vertices and is used to define node-level positional encodings, whereas the \(1\)-path Laplacian \(\mathbf{L}_1\) acts on directed edges and is used to define edge-level positional encodings. An explicit formula for \(\mathbf{L}_0\) is given by
\begin{equation*}
    \mathbf{L}_0(u,v) = \mathbf{D}(u,v)-\big( \mathbf{A}(u,v) + \mathbf{A}(v,u)\big),
\end{equation*}
where \(\mathbf{D}\) is the diagonal degree matrix associated with \(\mathbf{A}+\mathbf{A}^{\top}\), and $\mathbf{A}$ is the adjacency matrix defined by \(\mathbf{A}(u,v)=1\) if \((u,v)\in \vec E\), and zero otherwise. 
We note that if $G$ has no bidirectional edges, then $\mathbf{L}_0$ is equivalent to the traditional (combinatorial) graph Laplacian on the symmetrized graph associated with $G$ (constructed by placing an (undirected) edge between $u$ and $v$ if $(u,v)\in\vec{E}$ or $(v,u)\in\vec{E}$).
Moreover, the edge Laplacian is given by
\begin{equation}\label{eq:edge_lap}
    \mathbf{L}_1= \mathbf{B}_2\mathbf{B}_2^{\top} +
    \mathbf{B}_1^{\top}\mathbf{B}_1.
\end{equation}
The edge-level operator incorporates both incidence information through the down Laplacian and directed two-step path information through the up Laplacian. These encodings can then be incorporated into graph neural networks or graph Transformers as additional structural features, allowing the model to access directional higher-order information beyond the symmetrized adjacency structure.

\input{figures/motifs}

\paragraph{Beyond edge directionality.}
We highlight here a feature of the path Laplacian that makes it a natural backbone for extracting positional encodings. As discussed above, two-step relations are encoded in the edge Laplacian \(\mathbf{L}_1\) through the up-Laplacian term \(\mathbf{B}_2\mathbf{B}_2^{\top}\); see \Eqref{eq:edge_lap}. Moreover, \citet[Proposition 2.9]{grigor2014homotopy} shows that the space \(\Omega_2\)  of $\partial$-invariant allowed \(2\)-paths admits a basis spanned by the directed motifs in Figure~\ref{fig:directed-motifs}.  
Thus, when the boundary map \(\partial_2: \Omega_2 \rightarrow \Omega_1\) is represented in such a basis, the matrix \(\mathbf{B}_2\) encodes the incidence relations between edges and directed \(2\)-path motifs. Although we do not explicitly fix this basis in practice, this perspective suggests that the up Laplacian $\mathbf{B}_2\mathbf{B}_2^{\top}$ captures structural information beyond edge directionality alone.

\paragraph{A basis for $\Omega_2$.} 
The space of $p$-paths admits a decomposition $\Lambda_p = \mathcal{A}_p \oplus \mathcal{D}_p$ into the direct sum of 
$\mathcal{A}_p$ and the vector space $\mathcal{D}_p$ spanned by the non-allowed elementary $p$-paths. Then, for any $p\geq 1$, the restriction of the boundary map defined in \cref{eq:bdry_Lambda} to the space of allowed paths $\partial_p : \mathcal{A}_p \rightarrow \Lambda_{p-1}$ can be expressed as a sum of two maps $\partial_p = \partial_{p}^{\mathcal{A}} + \partial_{p}^{\mathcal{D}}$ defined by 
\(  \partial_{p}^{\mathcal{A}}
=
\pi_{p-1}^{\mathcal{A}}\circ\partial_p\), and \(
\partial_{p}^{\mathcal{D}}
=
\pi_{p-1}^{\mathcal{D}}\circ\partial_p,
\)
where $\pi_{p-1}^{\mathcal{A}}$ and $\pi_{p-1}^{\mathcal{D}}$ denote the projection maps onto $\mathcal{A}_{p-1}$ and $\mathcal{D}_{p-1}$, respectively. The vector space of $\partial$-invariant paths is precisely
\(
\Omega_p = \ker \partial_{p}^{\mathcal{D}}.
\)
The matrix representation $\mathbf{B}_p$ of the restricted boundary map $\partial_p : \Omega_p \rightarrow \Omega_{p-1}$ depends on choices of bases for $\Omega_p$, and $\Omega_{p-1}$. Constructing such bases is a key step in the construction of the path Laplacian in \Eqref{eq:PathLap}. 

In general, computing a kernel basis for large matrices can be
time-consuming. However, since we are mainly interested in the case
$p\leq 2$, this can be significantly simplified. The following theorem provides an explicit basis for the space of $\partial$-invariant paths $\Omega_2$. Algorithm~\ref{alg:omega2}, which implements this construction, is provided in Appendix~\ref{app:proofs}. 
We note that unlike the basis obtained by \citet[Proposition 2.9]{grigor2014homotopy}, our basis is orthonormal which allows for efficient numerical implementation.


\begin{theorem}\label{thm:basis}
Consider a directed graph $G$ with vector space of allowed $2$-paths
$\mathcal{A}_2$. Let $\mathcal{B}_{\mathcal{A}}^{(2)}$ denote the basis of $\mathcal{A}_2$
formed by all elementary allowed $2$-paths. Then $\Omega_2$ admits an orthonormal 
basis consisting of the $\partial$-invariant elementary $2$-paths
$\sigma \in \mathcal{B}_{\mathcal{A}}^{(2)}$, together with vectors of
the form
\begin{equation}
    \mathbf h_\ell
    =
    \frac{1}{\sqrt{\ell(\ell+1)}}
    \left(
        \sum_{q=1}^{\ell}\sigma_q
        -
        \ell \sigma_{\ell+1}
    \right),
\end{equation}
for $\ell=1,\ldots,m-1$, where, for each non-allowed elementary
$1$-path $\tau$, $\sigma_1,\ldots,\sigma_m$ is a fixed ordering of all
elementary allowed $2$-paths satisfying $\partial_2^{\mathcal{D}}(\sigma_q)=-\tau$.
\end{theorem}

\cref{thm:basis} allows an efficient construction of the boundary matrix $\mathbf{B}_2$, hence of the path Laplacian. 

 
\begin{figure}[htbp]
    \centering
    \resizebox{1.0\textwidth}{!}{
        \input{figures/pipeline/pathlappe}
    }
    \caption{Overview of PathLapPE construction and feature extraction pipeline.}
    \label{fig:placeholder}
\end{figure}

\subsection{EdgeSPE: Sign- and basis-invariant edge-level Path Laplacian PEs}\label{sec:EdgeSPE}
In this section, we briefly review  
Stable and Expressive Positional Encodings (SPE),  an effective encoding strategy introduced in \citet{huang2024stability} to make spectral PEs sign- and basis-invariant.  
Standard SPE was developed for node-level encodings and therefore cannot be directly applied to the edge-level PathLapPE encodings corresponding to $\mathbf{L}_1$ without suitable modifications.
We therefore develop a modified method, EdgeSPE,  which extends SPE to edge-level encodings. 

Letting $\mathbf{V}\in\mathbb{R}^{n\times d}$ be a matrix containing the first $d$ eigenvectors of a Laplacian matrix $\mathbf{L}$ and letting  $\lambda\in\mathbb{R}^d$ 
denote a vector of the corresponding eigenvalues, the original 
SPE is defined by                          \begin{equation}\label{eq:NodeSPE}
    \mathrm{SPE}(\mathbf{V},\lambda)=\rho\big(\mathbf{V}\mathrm{diag}(\phi_1(\lambda))\mathbf{V}^\top,\ \ldots,\ \mathbf{V}\mathrm{diag}(\phi_m(\lambda))\mathbf{V}^\top\big),  
\end{equation}     
where $\rho:\mathbb{R}^{n\times n\times m}\to\mathbb{R}^{n\times p}$ and each of the $\phi_\ell:\mathbb{R}^d\to\mathbb{R}^d$ are permutation equivariant functions.
Intuitively, the  $\phi_\ell$ act as a soft, learnable partition of the spectrum, in contrast to the hard partition into eigenspaces used by BasisNet \citep{lim2023sign}. 
See Appendix \ref{app:SPE} for further discussion.


\paragraph{EdgeSPE for PathLapPE.} 
To extend SPE to the edge-level setting, we let $\mathbf{U}\in\mathbb{R}^{|E|\times d}$ be a matrix containing the first $d$ eigenvectors of $\mathbf{L}_1$ and let  $\lambda^{(1)}\in\mathbb{R}^d$ denote a vector of the corresponding eigenvalues. 
Analogously to the kernels $\mathbf{V}\mathrm{diag}(\phi_i(\lambda))\mathbf{V}^\top$, used in the original SPE (see \cref{eq:NodeSPE}), we form edge-level SPE kernels $\mathbf{K}_i = \mathbf{U}\mathrm{diag}\big(\phi_i(\lambda^{(1)})\big)\mathbf{U}^\top \in \mathbb{R}^{|E|\times|E|}$, where $1 \leq i \leq m$.                For each edge $e$, we let $\mathcal{S}(e)=\{f:\mathrm{tail}(f)=\mathrm{head}(e)\}$ and
$\mathcal{P}(e)=\{f:\mathrm{head}(f)=\mathrm{tail}(e)\}$ denote its successors and its predecessors. Then for all $i$, we compute three encodings given by $\mathbf{h}^{\mathrm{diag}}_i(e) = \mathbf{K}_i(e,e)$, $
\mathbf h^{\mathrm{succ}}_i(e) = \sum_{f\in\mathcal{S}(e)} \mathbf{K}_i(e,f)$, and $\mathbf h^{\mathrm{pred}}_i(e) = \sum_{f\in\mathcal{P}(e)} \mathbf{K}_i(f,e)$.
Finally, we define the output of EdgeSPE by 
$\textbf{EdgeSPE}(e)=\textbf{MLP}\big(\mathbf h^{\mathrm{diag}}(e),\mathbf h^{\mathrm{succ}}(e),\mathbf h^{\mathrm{pred}}(e)\big)$.

             


\section{Experiments}\label{sec:experiments}

We evaluate PathLapPE as a positional encoding on a wide array of directed graph settings, including synthetic and real-world benchmarks, regression and classification targets, small and large-scale datasets, node- and graph-level tasks. We additionally analyze the computational cost of PathLapPE. All datasets are directed, and all relevant statistics are provided in Tables~\ref{tab:dataset_stats_node} and~\ref{tab:dataset_stats_graph} (Appendix~\ref{app:dataset_stats}). Additional details on training and model hyperparameters for all tasks are provided in Appendix~\ref{app:hparams}.

\begin{wrapfigure}[14]{r}{0.48\textwidth}
    \vspace{-15pt}
    \centering
    \includegraphics[width=.48\textwidth]{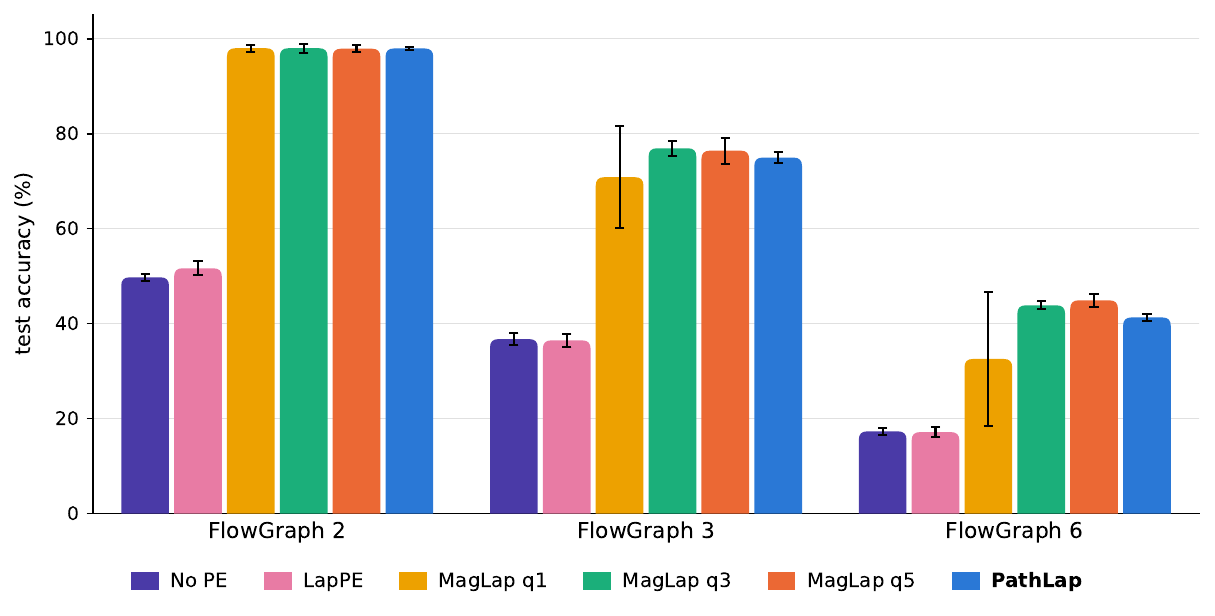}
    \caption{\textsc{FlowGraph} experiments, mean + std for five runs. Only directional PEs (PathLapPE and Multi-$q$-MagLapPE) perform above random.}
    \label{fig:flowgraph}
\end{wrapfigure}%

We compare PathLapPE with (i) no PEs, (ii) LapPE and (iii) RWPE~(both \citet{dwivedi2022graph}), and (iv) Multi-$q$-MagLapPE PEs~\citep{huang2025good}, using a variety of both conventional and direction-aware GNNs. 

In PathLapPE computation, we remove all self-loops, and pad zeros when $k > |V|$ or $|E|$.
In addition to uni- and bidirectional GAT and GINE models, we evaluate with BiGCNE, a model we introduce in Appendix~\ref{app:bigcne} as an extension of bidirectional GCN that supports edge features that act through a gating mechanism. Moreover, we process all node-level spectral positional encodings using the generalized SPE model introduced in \citet{huang2025good}, which yields invariant and stable representations. Similarly, we process edge-level PathLapPE encodings with our EdgeSPE, and report results both with and without EdgeSPE.  

\subsection{Synthetic flow graphs}
We begin with a simple task based on synthetic flow graphs designed to probe the directional information in PEs, primarily serving as a sanity check. 
\textsc{FlowGraph} datasets~\citep{wang2024directed} are generated by creating clusters of nodes that are then sequentially ordered with a probabilistic directional flow defined between consecutive clusters: $n\%$ of the edges flow in one direction (e.g., from cluster $c$ to $c+1$), whereas $(100-n)\%$ flow in the opposite direction. Classes are defined by $n\%$ so that the task is directly tied to learning edge directionality. After computing positional encodings, we convert the graph to undirected (i.e., $A(u,v)\leftarrow \max\{A(u,v),A(v,u)\}$) and use a standard 
(i.e., unidirectional) GINE to force the network to rely only on the PE to solve the task. 

\paragraph{Results.} We evaluate on \textsc{FlowGraph} $K \in \{2, 3, 6\}$, where $K$ denotes the number of classes defined in each task,
and present the results in \Cref{fig:flowgraph}. As expected, No PE and LapPE performance are essentially random, as the GINE layers have no directional information to work with. Multi-$q$-Laplacian PEs and PathLapPE can nearly fully solve the task for $K=2$, and perform well above random for $K \in \{3, 6\}$, indicating that they capture directionality. We observe that a single $q$-parameter is more unstable than PathLapPE, which the Multi-$q$ variants alleviate.

\subsection{\textsc{WebKB} datasets}

We next consider the WebKB datasets \textsc{Cornell}, \textsc{Wisconsin}, and \textsc{Texas}~\citep{Pei2020Geom-GCN:}, in which nodes represent web pages and directed edges represent hyperlinks. Node features consist of bag-of-words representations. The task is to classify the nodes as student, project, course, staff, or faculty. 
We compare PathLapPE with and without EdgeSPE against LapPE, RWPE and Multi-$q$-MagLapPE. For the latter, we report the best results over \(q \in \{0.01, \ 0.05, \ 0.10\}\).
Additional information on tuning and selected hyperparameters is provided in Appendix~\ref{app:webkb}, \Cref{tab:hparams}. 


\paragraph{Results.}
In Table~\ref{tab:WebKB}, we report mean test accuracy over the ten splits from \citet{Pei2020Geom-GCN:}. PathLapPE improves over the no PE baseline across all architectures and datasets and achieves the best average rank for each model.
The gains are largest for BiGINE and BiGCNE, reaching \(22.2\) and \(18.4\) percentage points, respectively, compared with at most \(7.6\) points for GAT. The larger gains for BiGINE and BiGCNE may reflect their ability to incorporate edge features directly into message passing. EdgeSPE further improves performance in most settings, indicating that the direction-aware edge encodings provide complementary information beyond the node-level PathLapPE.

\input{tables/webkb}

\subsection{Large-scale evaluation on \textsc{Roman-Empire}}

The WebKB datasets are relatively small and known to be noisy; thus, we next evaluate on a \emph{directed} variant of \textsc{Roman-Empire}~\citep{platonov2023a}, a much larger dataset with 22,662 nodes, derived from the Roman Empire Wikipedia page. 
Each node corresponds to a word; an edge exists between two nodes if one follows the other in the text, or the words are syntactically connected in a sentence. Node features are based on word embeddings. The task is to predict the syntactic role of each word. 


In addition to previous models, we evaluate bidirectional GAT and \emph{BiGINE-shared}, a bidirectional GINE that shares the MLP weights of the GINE convolutions across the two directions.
Weight sharing forces the model to rely more on the edge features, i.e., the spectral PEs, rather than on model expressivity 
to extract directional information, allowing us to evaluate the impact of individual PEs more distinctively.
We evaluate two Multi-$q$-MagLapPE configurations: the best single \(q\)-value of \(0.05\), and three stacked $q$ values.
For model architectures and hyperparameter selection, we largely follow \citet{platonov2023a}; full experimental details are provided in Appendix~\ref{app:roman_empire}, \Cref{tab:hparams}. 

\paragraph{Results.}

\input{tables/directed_roman_empire}

\Cref{tab:directed_roman_empire} reports the performance of PathLapPE compared to No PE, LaplacianPE, and Multi-$q$-MagLapPE baselines on \textsc{Roman-Empire} across five GNN architectures. PathLapPE achieves the best (or approximately equal-best, for BiGAT) result for every architecture. The gap between PathLapPE and the baselines is more pronounced for bidirectional GINE and GCNE, and is particularly large for our BiGINE-shared model, which is structurally restricted to rely on edge features for directionality. We also note that BiGINE-shared benefits more from PEs in general than the other architectures. 
Interestingly, both conventional GAT and its bidirectional variant benefit the least from (i) PEs in general, and (ii) PathLapPE specifically, with most PEs falling within one standard deviation of the PathLapPE results. 


\subsection{Open Circuit Benchmark \& High-Level Synthesis}

Finally, we consider inductive, graph-level regression tasks on the \textsc{Open Circuit Benchmark (OCB)} and \textsc{High-Level Synthesis (HLS)} datasets. For each task, we follow the implementations and optimized hyperparameters for BiGINE and GINE from~\citet{huang2025good}.

\textbf{OCB}~\citep{dong2023cktgnn} consists of 8,026 operational amplifier-based (op-amp) analog circuits as directed graphs, and the task is to predict DC gain (Gain), bandwidth (BW) and phase margin (PM) for each circuit. We evaluate two Multi-$q$-MagLapPE configurations, each with ten $q$ values but with phase parameters of 0.01 and 0.05. 

\textbf{HLS}~\citep{wu2022high} represents 18,570 C/C++ programs as directed intermediate representation (IR) graphs for performance analysis on FPGA integrated circuits. The task is to predict digital signal processor (DSP) and look-up table (LUT) usage, and critical path timing slack (CP, permitted delay that does not affect overall program runtime) for each program. We evaluate two Multi-$q$-MagLapPE configurations, each with five $q$ values but with phase parameters of 0.01 and 0.1. 

\input{tables/amp_hsl}

\paragraph{Results.}

\Cref{tab:amp_hsl} presents test results for the six tasks across both benchmarks, and the average rank of each PE. PathLapPE with EdgeSPE is the best method on average by some margin with average ranks of 1.50 and 1.83 for BiGINE and GINE respectively, as well as the most consistent: of the twelve model $\times$ task combinations evaluated, it is the best method on six, and top-2 on eleven of them -- the sole remaining task being GINE $\times$ CP, where PEs are not helpful in general. 
However, PathLapPE \emph{without} EdgeSPE (where we simply use a linear mapping of the raw edge-PEs) is relatively inconsistent, and falls behind not only the EdgeSPE variant but also other PEs on several tasks. The benefits of EdgeSPE are thus more apparent in such inductive settings where learning approximate sign- and basis-invariance is valuable, whereas in transductive node-level tasks one can precompute PEs on a single basis, as all training and inference happens over a single graph. Finally, we noticed that Multi-$q$-MagLapPE suffered from instabilities with GINE, where some runs converged to reasonable losses, while others fully diverged over the course of training. We provide a more extensive discussion in Appendix~\ref{app:ocb_hls}.



\begin{figure}[t]
  \centering
  \includegraphics[width=0.49\textwidth]{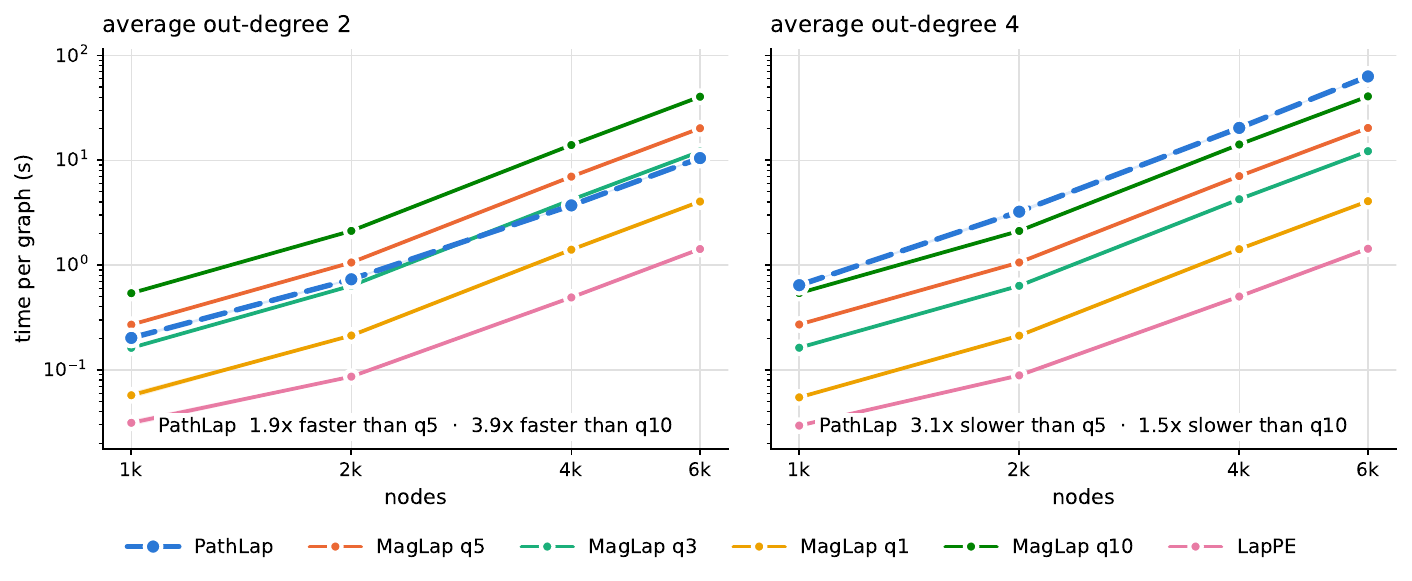}\hfill
  \includegraphics[width=0.49\textwidth]{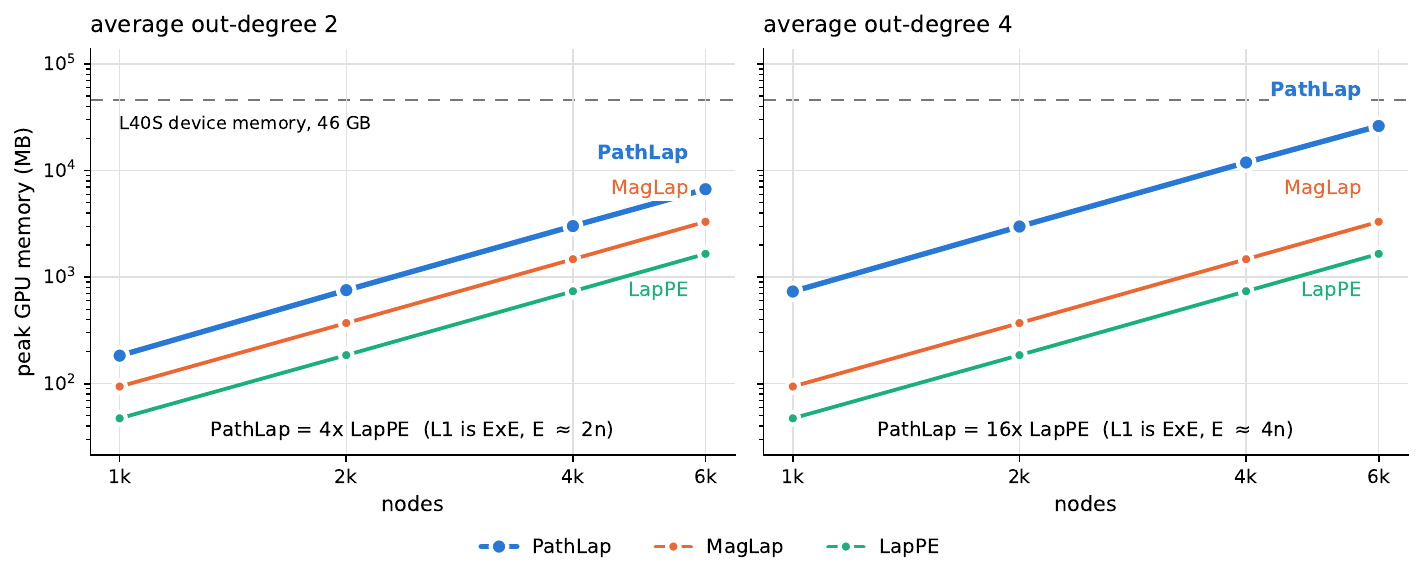}
  \caption{Time (left pair) and memory (right pair) cost of PE computation on directed ER graphs on an NVIDIA L40S GPU. Dense solver on GPU, $k=8$, three repeats for each experiment.}
  \label{fig:ER_cost}
\end{figure}

\subsection{Computational cost of PathLapPE}

PathLapPE has different scaling properties than other spectral encodings: the $\mathbf{L}_0$ and $\mathbf{L}_1$ path Laplacians scale with $|V|$ (like LapPE) and $|E|$, respectively, so PathLapPE computation is more efficient on sparser graphs. In \Cref{fig:ER_cost}, we grow Erd\H{o}s--R\'enyi graphs at two fixed densities (average out-degree two and four) and measure PE computation time and GPU memory. On the sparser graphs, PathLapPE is faster than Multi-$q$-MagLapPE for $\#q\geq3$ at scale. 
On the denser graphs, its cost is just above that of $\#q=10$.
PathLapPE uses more memory than Multi-$q$-MagLapPE, which processes $q$ values sequentially at fixed memory cost. For large graphs, sparse eigensolvers~\citep{scipy} can reduce this memory at a modest cost in time. Appendix~\ref{app:pe_cost}, \Cref{tab:pe-cost-horizontal} reports time and memory on our benchmarks; on the largest, \textsc{Roman-Empire} (22,662 nodes), sparse PathLapPE takes three seconds and 70 MB, making PE computation costs negligible compared to training. Overall, the additional computational cost of PathLapPE relative to Multi-$q$-LapPE is modest, especially considering the additional benefit that PathLapPE requires no tuning.




\section{Conclusion}
We introduced PathLapPE, a spectral positional encoding for directed graphs derived from the Laplacian of the path complex. PathLapPE provides real-valued node- and edge-level encodings that capture edge orientation and directed 2-path structure, without the computational burden of higher-order message-passing. We additionally provide EdgeSPE, which further improves the performance of edge-level encodings.    
Across a wide range of node- and graph-level
directed benchmarks, PathLapPE performs as well as or better than other
spectral PEs, with the largest gains for direction-aware architectures. Unlike
magnetic Laplacian encodings, it has no phase parameter to tune and requires only
one real eigendecomposition per Laplacian, making it computationally efficient.
Extending our PE to path Laplacians of higher order and reducing memory cost on denser graphs, as well as the inclusion into graph Transformers, are natural directions for future work.



\subsection*{Acknowledgements and Disclosure of Funding}

This work was partially funded by Bourse en intelligence artificielle des Études supérieures et postdoctorales (ESP) 2023--2024, Bourse d’excellence du DIRO (2025) [S.C.], and Bourse de fin d’études doctorales (2026--2027) [S.C. and L.M.]; NSF OIA 2242769 [M.P.]; Canada CIFAR AI Chair, IVADO (Institut de valorisation des données) grant PRF-2019-3583139727, FRQNT (Fonds de recherche du Québec -- Nature et technologies) grant 299376, and NSERC Discovery Grant 03267 [G.W.]; and NSF DMS grant 2327211 [M.P. and G.W.]. This work has received funding from the Swiss State Secretariat for Education, Research and Innovation~(SERI) [B.R.]. This research was also enabled in part by computing resources provided by Mila (mila.quebec). The content provided here is solely the responsibility of the authors and does not necessarily represent the official views of the funding agencies.


\subsection*{AI use statement}

Generative AI usage in this work has been strictly limited to light text editing, adjustments of provided figures (e.g., TikZ adjustments for Fig. 2), and writing minor boilerplate code. Generative AI tools have \emph{not} been used for the development of the mathematical framework nor the accompanying formulation or proving of mathematical claims, experiment design, dataset generation, data analysis nor core development of the codebase. We have, of course, reviewed all AI-assisted work. LLM adjustments to any text and figures have been checked by multiple authors; any LLM-generated code adjustments have also been verified and tested for correctness. We take responsibility for the final content of this work, including text, claims or artifacts produced with the aid of generative AI.



\bibliography{iclr2027_conference}
\bibliographystyle{iclr2027_conference}

\clearpage


\appendix

\section{Limitations}
\label{sec:limitations}

In this work, we have considered path Laplacians only up to order $p=2$; higher-order constructions are possible, albeit at increasing computational costs. In conjunction with this point, PathLapPE construction and eigendecomposition is inherently more memory-intensive on dense graphs due to its scaling with respect to the number of edges. While efficient workarounds exist particularly for eigendecomposition, extending to even higher-order paths would exacerbate these memory limitations. Our evaluation also largely focuses on uni- and bidirectional GNNs; while we expect the observed performance gains to naturally extend to graph Transformers, further evaluations on Transformer-based architectures would solidify this claim. Finally, in our research we have observed that there is limited work on whether directionality is truly beneficial to a given task on directed benchmarks; future work on investigating types of directed graphs that would benefit from PathLapPE would increase the potential impact of our method.


\section{The magnetic Laplacian}

For undirected graphs, the standard graph Laplacian provides a canonical self-adjoint operator whose eigenvectors encode global graph structure. In the directed setting, however, there is no unique analogue of the graph Laplacian. Different constructions preserve different aspects of directionality. In this section, we briefly describe magnetic Laplacians, which encode directionality through complex phases.

Let \(G=(V,\vec E)\) be a directed graph with adjacency matrix \(\mathbf{A}\), where \(\mathbf{A}_{uv}=1\) if \((u,v)\in \vec E\) and \(\mathbf{A}_{uv}=0\) otherwise. Magnetic Laplacian constructions first define a symmetrized adjacency matrix \(\mathbf{A}_s = \frac{1}{2}(\mathbf{A}+\mathbf{A}^{\top})\), which records the undirected support of the graph, possibly with weights reflecting whether one or both orientations are present. Directionality is then incorporated through a skew-symmetric phase matrix \(\Theta^{(q)} = 2\pi q(\mathbf{A}-\mathbf{A}^{\top})\), where \(q\) is a tunable phase parameter controlling the strength of the directional signal. The corresponding complex Hermitian adjacency matrix is given by
\[
\mathbf{H}^{(q)} = \mathbf{A}_s \odot \exp\!\left(i\Theta^{(q)}\right),
\]
where \(\odot\) denotes entrywise multiplication. The degree matrix associated with \(\mathbf{A}_s\) is given by \(\mathbf{D}_s(u,u) = \sum_{v\in V} \mathbf{A}_s(u,v)\). The unnormalized magnetic Laplacian is then defined by
\[
\mathbf{L}^{(q)}_{\mathrm{mag}} = \mathbf{D}_s - \mathbf{H}^{(q)}.
\]
The normalized version
\[
\mathcal{L}^{(q)}_{\mathrm{mag}} = \mathbf{I} - \mathbf{D}_s^{-1/2} \mathbf{H}^{(q)} \mathbf{D}_s^{-1/2}
\]
is often used in graph learning.

\section{Proof of Theorem 1} \label{app:proofs}
In this section, we provide a proof of Theorem \ref{thm:basis}. We recall that any directed graph $G = (V, \vec E)$ gives rise to a vector space $\Lambda_p$ spanned by elementary $p$-paths, a vector space $\mathcal{A}_p \subseteq \Lambda_p$ spanned by allowed elementary $p$-paths and finally the main object of interest in this section, the vector space $\Omega_p \subseteq \mathcal{A}_p$ of $\partial_p$-invariant allowed paths. In this section, $p=1,2$, and we are interested in constructing a basis for $\Omega_2$ in closed form.  Before proving Theorem \ref{thm:basis}, we first consider the following example, in which we compute
a basis for $\Omega_2$ in a simple case.

\begin{example} \label{ex:example3}
As in Example \ref{ex: directed triangle under stand allowed paths}, consider the directed triangle represented in Figure~\ref{fig:directed-motifs}(a) and the elementary paths $u=ab, v=bc$, $w=ac$, and  
 $t=abc$. Then, an ordered basis for $\Omega_0$ is given by the vertices $a, b$, and $c$ and an ordered basis for $\Omega_1$ is given by the $1$-paths corresponding to the directed edges $u$, $v$, and $w$. Since $\partial t = bc - ac + ab=v-w+u$, we conclude that $t$ is a $\partial$-invariant allowed path, hence an element of $\Omega_2$. One may also verify that there are no other independent elements of $\Omega_2$ and so $\Omega_2$ is a one-dimensional vector space with basis $\{t\}$. Given these choices of bases (in the order listed above), the boundary matrices are given by 
\begin{equation*}
    \mathbf B_1 = 
    \begin{bmatrix}
        -1 & 0 & -1 \\
        1 & -1 & 0 \\
        0 & 1 &  1 
    \end{bmatrix} ; \quad \mathbf B_2 = \begin{bmatrix}
        1& 1 & -1
    \end{bmatrix}^\top.  
\end{equation*}

The unweighted edge Laplacian is therefore $ \mathbf L_1 = \mathbf B_2 \mathbf  B_{2}^\top + \mathbf B_{1}^\top \mathbf B_1 = 3\mathbf I_{3}$.

\end{example}

We recall from \Cref{subsec:PathLapPEs} that $\Lambda_1$ can be expressed as the direct sum $\mathcal{A}_1 \oplus \mathcal{D}_1$ and the corresponding projection maps $\pi^{(1)}_{\mathcal{A}} : \Lambda_1 \rightarrow \mathcal{A}_1$, and $\pi_{\mathcal{D}}^{(1)} : \Lambda_1 \rightarrow \mathcal{D}_1$ induce the maps 
\begin{equation*}
    \partial_{2}^{\mathcal{A}}: \mathcal{A}_2 \rightarrow \mathcal{A}_1 ; \quad \partial_{2}^{\mathcal{A}}(\sigma)\mapsto \pi_{\mathcal{A}}^{(1)} \circ \partial_2 (\sigma),  
\end{equation*}
and 
\begin{equation}\label{eq:bdry_D}
    \partial_{2}^{\mathcal{D}}: \mathcal{A}_2 \rightarrow \mathcal{D}_1 ; \quad \partial_{2}^{\mathcal{D}}(\sigma)\mapsto \pi_{\mathcal{D}}^{(1)} \circ \partial_2 (\sigma).    
\end{equation}
Here $\partial_2 : \Lambda_2 \rightarrow \Lambda_1$ is the usual boundary map on $p$-paths. Note that, by definition,  $\Omega_2 $ is equal to $\ker \partial_{2}^{\mathcal{D}}$ since $\sigma$ is in $\Omega_2$ if and only if $\partial_2(\sigma)\in\mathcal{A}_1$ and $\mathcal{D}_1$ is the orthogonal complement of $\mathcal{A}_1$. 

We consider the basis of $\Lambda_p$ consisting of elementary $p$-paths, and fix an ordering for the basis, denoted by:  
\begin{equation*}
    \mathcal{B}_{\Lambda}^{(p)}:= \lbrace \sigma_{i}^{(p)}|\;  i=1, \ldots, \dim(\Lambda_p)\rbrace 
\end{equation*}

The basis of $\mathcal{A}_p$ is the subset $\mathcal{B}_{\mathcal{A}}^{(p)} \subseteq
\mathcal{B}_{\Lambda}^{(p)}$ corresponding to the allowed paths. We similarly denote by $\mathcal{B}_{\mathcal{D}}^{(1)} \subseteq
\mathcal{B}_{\Lambda}^{(1)}$ the basis of $\mathcal{D}_1$ formed by the non-allowed elementary $1$-paths. Finally, we denote by
$\mathbf{B}_{2}^{\mathcal{D}}$ the matrix representation of the map $\partial_2^{\mathcal{D}}$ defined in \Cref{eq:bdry_D}, expressed with respect to the bases $\mathcal{B}_{\mathcal{A}}^{(2)}$ and
$\mathcal{B}_{\mathcal{D}}^{(1)}$.


\begin{proof}
Consider an arbitrary elementary allowed $2$-path
$\sigma_{j}^{(2)} = v_0v_1v_2$. By definition,
\begin{equation}
    \partial_2(\sigma_{j}^{(2)})
    =
    v_1v_2 - v_0v_2 + v_0v_1.
\end{equation}
Since $\sigma_{j}^{(2)}$ is an allowed path, the edges $v_0v_1$
and $v_1v_2$ are allowed $1$-paths (i.e., directed edges in $G$). If $v_0=v_2$, the elementary $2$-path $\sigma_{j}^{(2)}$ is $\partial$-invariant. Indeed, by definition of $\Lambda_1$ paths of the form $v_iv_i$ are identified with zero. Assume $v_0\neq v_2$ then, 
\begin{equation*}
    \partial_2^{\mathcal{D}}(\sigma_{j}^{(2)})
    =
    -\pi_{\mathcal{D}}^{(1)}(v_0v_2).
\end{equation*}
Hence, $\sigma_{j}^{(2)} \in \Omega_2$ if and only if
$v_0v_2 \in \mathcal{A}_1$. Thus, if $v_0v_2$ is an allowed $1$-path, then $\mathbf{B}_{2}^{\mathcal{D}}[:,j]=0$ and so $\sigma_j^{(2)}$ is $\partial$-invariant and an element of $\Omega_2$.
Otherwise, let
$\sigma_i^{(1)}=v_0v_2 \in \mathcal{B}_{\mathcal{D}}^{(1)}$.
Then, $\mathbf{B}_{2}^{\mathcal{D}}[i,j] = -1$. Assume there exist exactly $m$ elementary allowed $2$-paths
$\sigma_{j_1}^{(2)},\ldots,\sigma_{j_m}^{(2)}$ satisfying
\begin{equation*}
    \partial_2^{\mathcal{D}}
    \bigl(\sigma_{j_q}^{(2)}\bigr)
    =
    -\sigma_i^{(1)},
    \qquad q=1,\ldots,m.
\end{equation*}
Then, the restriction of $\mathbf{B}_{2}^{\mathcal{D}}[i,:]$ to these basis vectors corresponds to the constant row vector of length $m$ given by 
\begin{equation*}
    -\mathbf{1}^{\top}
    =
    (-1,\ldots,-1).
\end{equation*}
Since $\Omega_2 \simeq \ker \mathbf{B}_{2}^{\mathcal{D}}$, it remains
to construct a basis for the kernel of $-\mathbf{1}^{\top}$.
If $m=1$, this kernel is trivial. For $m>1$, consider the column
vectors in $\mathbb{R}^m$ given by
\begin{equation}
    \frac{1}{\sqrt{\ell(\ell+1)}}
    \left(
        \underbrace{1,\ldots,1}_{\ell},
        -\ell,
        0,\ldots,0
    \right)^\top,
    \qquad
    1\leq \ell<m.
\end{equation}
It is straightforward to verify that these vectors constitute an
orthonormal basis for $\ker(-\mathbf{1}^{\top})$. Writing them with
respect to the basis elements
$\sigma_{j_1}^{(2)},\ldots,\sigma_{j_m}^{(2)}$ gives the $2$-paths
\begin{equation*}
    \mathbf h_\ell
    =
    \frac{1}{\sqrt{\ell(\ell+1)}}
    \left(
        \sum_{q=1}^{\ell}\sigma_{j_q}^{(2)}
        -
        \ell\sigma_{j_{\ell+1}}^{(2)}
    \right),
    \qquad
    \ell=1,\ldots,m-1.
\end{equation*}

Finally, every column of $\mathbf{B}_{2}^{\mathcal{D}}$ is either
zero or contains exactly one nonzero entry. The nonzero columns can
therefore be partitioned according to the non-allowed elementary
$1$-path appearing in their boundary. Combining the
$\partial$-invariant elementary $2$-paths corresponding to the zero
columns with the vectors $\mathbf h_\ell$ constructed for each such
group gives a basis for $\ker \mathbf{B}_{2}^{\mathcal{D}}$, and
therefore for $\Omega_2$.
\end{proof}

\begin{remark}
The basis in \cref{thm:basis} is consistent with the structural description of
$\Omega_2$ given by \citet[Proposition~2.9]{grigor2014homotopy}
(see also \citet[Theorem~1.8]{grigor2022advances}), where $\Omega_2$ is
shown to be spanned by triangles, squares, and double arrows (i.e., $2$-paths in \cref{fig:directed-motifs}(a), (b), and (c), respectively). In our
construction, double arrows and triangles are precisely
$\partial$-invariant elementary $2$-paths. For a fixed non-allowed
elementary $1$-path $\tau$, the remaining basis vectors satisfy
\[
\mathbf h_\ell
=
\frac{1}{\sqrt{\ell(\ell+1)}}
\sum_{q=1}^{\ell}
(\sigma_q-\sigma_{\ell+1}),
\]
and therefore lie in the span of the corresponding squares (i.e., $2$-paths illustrated in \cref{fig:directed-motifs}~(b)). Thus,
\cref{thm:basis} can be viewed as an orthonormal refinement of this
spanning description.

The orthonormality is also useful for constructing the path Laplacian.
Indeed, the matrix formula
\[
\mathbf L_p
=
\mathbf B_{p+1}\mathbf B_{p+1}^{\top}
+
\mathbf B_p^{\top}\mathbf B_p
\]
assumes that the boundary matrices are represented with respect to
orthonormal bases. 

Let $\mathbf M_2$ denote the Gram
matrix of a chosen basis of $\Omega_2$ with respect to the induced
inner product. Since the natural bases of $\Omega_0$ and $\Omega_1$ (the vertex set and directed edge set, respectively)
are orthonormal, the matrix representation of the adjoint
$\partial_2^*$ is
\[
[\partial_2^*]
=
\mathbf M_2^{-1}\mathbf B_2^\top.
\]
Hence, for an arbitrary basis of $\Omega_2$,
\[
\mathbf L_1
=
\mathbf B_2\mathbf M_2^{-1}\mathbf B_2^\top
+
\mathbf B_1^\top\mathbf B_1.
\]
For the orthonormal basis of $\Omega_2$ given in
\cref{thm:basis}, $\mathbf M_2=\mathbf I$, recovering the simpler
expression
\[
\mathbf L_1
=
\mathbf B_2\mathbf B_2^\top
+
\mathbf B_1^\top\mathbf B_1.
\]

Our closed-form basis
therefore avoids an additional numerical orthogonalization or
Gram-matrix correction when constructing $\mathbf L_1$.

\end{remark}

\begin{algorithm}[h]
\caption{Construction of a basis for $\Omega_2$}
\label{alg:omega2}
\begin{algorithmic}[1]
\State \textbf{Input:} Directed graph $G=(V,\vec E)$
\State Construct the set $\mathcal{A}_2$ of allowed $2$-paths and $\mathbf{B}_{2}^{\mathcal{D}}$, the matrix representation of $\partial_2^{\mathcal{D}}$.
\State Initialize $\mathcal{B}_{\Omega_2}\gets\varnothing$ and
$\mathcal{R}\gets\varnothing$.

\For{$i=1,\ldots,|\mathcal{A}_2|$}
    \If{$\mathbf{B}_{2}^{\mathcal{D}}[:,i]=0$}
        \State Add $e_i$ to $\mathcal{B}_{\Omega_2}$.
    \Else
        \State Let $r$ be the unique row such that
        $(\mathbf{B}_{2}^{\mathcal{D}})_{r i}\neq 0$.
        \If{$r\notin\mathcal{R}$}
            \State Let
            \[
            J_r=\{j_1,\ldots,j_s\}
            :=
            \{j:(\mathbf{B}_{2}^{\mathcal{D}})_{rj}\neq 0\}.
            \]
            \For{$\ell=1,\ldots,s-1$}
                \State Add to $\mathcal{B}_{\Omega_2}$
                \[
                h_\ell
                =
                \frac{1}{\sqrt{\ell(\ell+1)}}
                \left(
                \sum_{q=1}^{\ell} e_{j_q}
                -
                \ell e_{j_{\ell+1}}
                \right)
                \]
            \EndFor
            \State $\mathcal{R}\gets\mathcal{R}\cup\{r\}$.
        \EndIf
    \EndIf
\EndFor
\State \Return $\mathcal{B}_{\Omega_2}$.
\end{algorithmic}
\end{algorithm}


\section{Further Discussion on SPE}\label{app:SPE}

As shown in \citet{huang2024stability}, SPE is invariant to the choice of eigenbasis and in particular is invariant to sign flips. We briefly review why that is the case.


Let $\mu_1 < \ldots < \mu_r$ denote the distinct eigenvalues among $\lambda_1,\ldots,\lambda_d$. For each $1\leq j\leq r$, we let $I_j = \{\, i : \lambda_i = \mu_j \,\}$ and let $\mathbf{V}_j \in \mathbb{R}^{n \times |I_j|}$ denote the submatrix of $\mathbf{V}$ whose columns are the corresponding eigenvectors (i.e., the $\mathbf{V}_i$, $i \in I_j$). Writing the kernel $\mathbf{Q}_\ell$ as a sum of rank-one terms yields 
\begin{align*}
    \mathbf{Q}_\ell &=\; \mathbf{V}\,\mathrm{diag}(\phi_\ell(\lambda))\,\mathbf{V}^\top
    =\; \sum_{i=1}^{d} \phi_\ell(\lambda)_i \, \mathbf{V}_i \mathbf{V}_i^\top.
\end{align*}       
Let $i$ and $i'$ be two indices with  $\lambda_i = \lambda_{i'}$, and let $\pi$ be the transposition $\pi$ of $i$ and $i'$. By construction, $\pi\lambda = \lambda$, so permutation equivariance gives $$\phi_\ell(\lambda) = \phi_\ell(\pi\lambda) = \pi\,\phi_\ell(\lambda),$$ 
i.e., $\phi_\ell(\lambda)_i = \phi_\ell(\lambda)_{i'}$. Therefore,  letting $w_{\ell j}$ denote the unique value of $\phi_\ell(\lambda)_i$ for  $i\in I_j$, we may write                                  \begin{equation}                                                               \mathbf{Q}_\ell \;=\; \sum_{j=1}^{r} w_{\ell j}\, \mathbf{V}_j \mathbf{V}_j^\top .         \end{equation}
Each of the terms $\mathbf{V}_j\mathbf{V}_j^\top$ is simply a projection matrix onto the eigenspace of $\mu_j$. Projection does not depend on the choice of eigenbasis and so $\mathbf{Q}_\ell$ is invariant to the choice of basis (in particular it is invariant to sign flips). Therefore SPE is invariant as well. 

\citet{huang2024stability} also show, in Theorem 3.1, that if $\phi_\ell$ and $\rho$ are Lipschitz, then SPE is stable with respect to the input Laplacian $\mathbf{L}$ (with a Lipschitz constant that depends on the eigengap $\gamma=\lambda_{d+1}-\lambda_d$). Furthermore, in Proposition 3.2, it is shown that SPE can universally approximate any continuous basis invariant function. Finally, they also show that SPE makes effective use of the eigenvalues (in addition to the eigenvectors) and can count cycles in Propositions 3.3 and 3.4.  


\section{BiGCNE: Directed Gated GCN with Node and Edge Positional Encodings}\label{app:bigcne}

PathLapPE provides positional information at both the node and edge levels. To leverage our direction-aware edge PE while maintaining a fair comparison with other node-level PEs, 
we use an edge-gated extension of bidirectional message passing, which we denote by \textbf{BiGCNE}. Following the Dir-GNN framework~\citep{rossi2024edge}, incoming information and outgoing information are aggregated separately. In addition, the contribution of each neighboring node is modulated by a learned gate that depends on the positional encoding of the corresponding directed edge.

Let $\mathbf{x}_i$ denote the input features of node $i$, let $\mathbf{p}^{\mathrm{node}}_i$ denote its node-level positional encoding, and let $\mathbf{p}^{\mathrm{edge}}_{ij}$ denote the edge-level positional encoding associated with the directed edge $(i,j)\in\vec E$. The node-level positional encoding is incorporated into the initial node representation by concatenation, that is, $ \tilde{\mathbf{x}}_i = \bigl[ \mathbf{x}_i \,\|\, \mathbf{p}^{\mathrm{node}}_i \bigr]$. Moreover, we let $\mathbf{h}^{(0)}_i$ denote the resulting initial hidden representation after the input embedding. At each BiGCNE layer $\ell$, messages from incoming and outgoing neighbors are computed separately. For an incoming edge $(j,i)\in\vec E$, the corresponding gate is given by \[ \gamma_{ji}^{\mathrm{in},(\ell)} = \sigma\!\left( \bigl(\mathbf{w}_{g,\mathrm{in}}^{(\ell)}\bigr)^\top \mathbf{p}^{\mathrm{edge}}_{ji} + \mathbf b_{g,\mathrm{in}}^{(\ell)} \right), \]
where $\mathbf{w}_{g,\mathrm{in}}^{(\ell)}$ and $\mathbf b_{g,\mathrm{in}}^{(\ell)}$ are learnable weights and biases, respectively. The incoming messages are then obtained by 
\[ \mathbf{m}_{i}^{\mathrm{in},(\ell)} = \operatorname{AGG}_{(j,i)\in\vec E} \left( \gamma_{ji}^{\mathrm{in},(\ell)} \mathbf{W}_{\mathrm{in}}^{(\ell)} \mathbf{h}_{j}^{(\ell)} \right). \]
Similarly, for an outgoing edge $(i,j)\in\vec E$, we define the gate 
\[ \gamma_{ij}^{\mathrm{out},(\ell)} = \sigma\!\left( \bigl(\mathbf{w}_{g,\mathrm{out}}^{(\ell)}\bigr)^\top \mathbf{p}^{\mathrm{edge}}_{ij} + \mathbf b_{g,\mathrm{out}}^{(\ell)} \right), \] and the outgoing messages by
\[ \mathbf{m}_{i}^{\mathrm{out},(\ell)} = \operatorname{AGG}_{(i,j)\in\vec E} \left( \gamma_{ij}^{\mathrm{out},(\ell)} \mathbf{W}_{\mathrm{out}}^{(\ell)} \mathbf{h}_{j}^{(\ell)} \right). \]
Here, $\sigma$ denotes the sigmoid function. We use scalar edge gates, so that each $\gamma_{ij}^{\mathrm{in},(\ell)}$ or $\gamma_{ij}^{\mathrm{out},(\ell)}$ controls the overall contribution of the corresponding message. The incoming and outgoing branches have independent learnable parameters. In our implementation, $\operatorname{AGG}$ is mean aggregation.
The current node representation and the two directional aggregates are then concatenated and combined through a learnable transformation: \[ \mathbf{h}_i^{(\ell+1)} = \operatorname{Dropout} \left( \operatorname{ReLU} \left( \mathbf{W}_{\mathrm{com}}^{(\ell)} \left[ \mathbf{h}_i^{(\ell)} \,\|\, \mathbf{m}_i^{\mathrm{in},(\ell)} \,\|\, \mathbf{m}_i^{\mathrm{out},(\ell)} \right] + \mathbf{b}_{\mathrm{com}}^{(\ell)} \right) \right). \]
When the input and output dimensions agree, a residual connection is added (i.e. $\mathbf{h}_i^{(\ell+1)} =\mathbf{h}_i^{(\ell+1)} + \mathbf{h}_i^{(\ell)}$).
Unlike the classical Dir-GNN, which combines incoming and outgoing messages using a predefined parameter $\alpha$, we use the edge-level positional encoding to modulate message passing through the learned gates at each layer. 
This direction-aware architecture preserves structural and directionality information captured through the PE that would otherwise be diluted by traditional architectures such as GCN. 
A directed gated encoder modulates information flow rather than simply appending the aggregated edge PEs to the node features.

\section{Dataset statistics}
\label{app:dataset_stats}
In Tables~\ref{tab:dataset_stats_node} and~\ref{tab:dataset_stats_graph}, we report basic statistics for the benchmark datasets considered in this paper for node-level and graph-level tasks, respectively. 
\input{tables/dataset_stats_node}
\input{tables/dataset_stats_graph}

\section{Experimental details \& Hyperparameters}
\label{app:hparams}

\subsection{Synthetic flow graphs}
\label{app:flowgraphs}

\textsc{FlowGraph} datasets~\citep{wang2024directed} are generated by creating clusters of nodes with random interconnections that are then sequentially ordered with a probabilistic directional flow defined between consecutive clusters: $n\%$ of the edges flow in one direction (e.g., from cluster $c$ to $c+1$), whereas $(100-n)\%$ flow in the opposite direction. Classes are defined by $n\%$ so that the task is directly tied to learning edge directionality, with each class assigned 1,000 samples. After computing positional encodings, we convert the graph to undirected and use a vanilla (i.e., unidirectional) GINE to force the network to rely only on the PE to solve the task. Hyperparameters for \textsc{FlowGraph} are provided in \Cref{tab:hparams}.

\subsection{WebKB datasets}
\label{app:webkb}

On all \textsc{WebKB} datasets, the number of message passing layers is set to two, and a grid search is performed over the following parameter values: dropout \(\in \{0.0, \ 0.5\}\), number of eigenvalues \(k_{\mathrm{node}}, k_{\mathrm{edge}} \in \{16, \ 32\}\), 
and \(q \in \{0.01, \ 0.05, \ 0.10\}\) for Multi-$q$-MagLapPE. Note that we do not stack the magnetic Laplacian encodings for several values of \(q\) here; the tested values of \(q\) and \(\alpha\) were selected based on the best results reported in \citet{zhang2021magnet} and \citet{rossi2024edge}. Hyperparameters for all \textsc{WebKB} datasets are provided in \Cref{tab:hparams}.

\subsection{Roman-Empire}
\label{app:roman_empire}

\textsc{Roman-Empire} is derived from the Wikipedia entry of the same name, where each node in the graph corresponds to a word in the text, not necessarily unique. An edge exists between two nodes if (i) one follows the other in the text, or (ii) the words are syntactically connected in a sentence. Node features are based on word embeddings, and the task is to predict the syntactic role of each word. Despite the semantic relationships in the graph implying directed edges, \citet{platonov2023a} provide an undirected graph. \citet{rossi2024edge} then provide a directed version using the same general construction but processing the graph to account for directionality; we evaluate our method on this directed version.

We evaluate two Multi-$q$-MagLapPE configurations: the best single \(q\)-value of \(0.05\), after searching over \(q \in \{0.05, 0.1, 0.15\}\), as well as a version that uses all three \(q\) values, following \citet{huang2025good} -- stacking PEs derived from multiple \(q\) values typically outperforms any single \(q\)-value. For all PEs, we search over \(k \in \{8, 16, 32, 64\}\), i.e., the number of nontrivial eigenvectors/eigenvalues, and report the best results. Table~\ref{tab:directed_roman_empire_complete} provides the full results for all \(k\) values.

Finally, since PathLapPE provides \emph{two} sets of encodings, node- and edge-level, we evaluate a more parameter-restricted version that maps each set of encodings to \emph{half} of the node/edge hidden dimension of 512 before concatenation for a fairer comparison with the baselines; we denote this variant PathLapPE (\(256 \times 2\)). For model architectures and hyperparameter selection, we largely follow \citet{platonov2023a}; full experimental details and selected hyperparameters are provided in \Cref{tab:hparams}. 

\paragraph{Complete results.} Here, we present the full set of results summarized in Table~\ref{tab:directed_roman_empire} of the main text. Specifically, we report results for all tested \(k\) values; the final column (\textbf{Best}) corresponds to Table~\ref{tab:directed_roman_empire}. The additional results demonstrate that (i) the superiority of PathLapPE is largely independent of the number of eigenvectors/eigenvalues considered, and that (ii) in almost all cases, the PEs are robust to the choice of \(k\), meaning that \(8\) or \(16\) is likely sufficient for \textsc{Roman-Empire}. Finally, restricting the dimensionality of PathLapPE encodings has no tangible drawback, as the two variants are always close and well within one standard deviation.

\input{tables/directed_roman_empire_complete}

\input{tables/hparams_exps}

\FloatBarrier

\subsection{Open Circuit Benchmark \& High-Level Synthesis}
\label{app:ocb_hls}

\textsc{Open Circuit Benchmark (OCB)} and \textsc{High-Level Synthesis (HLS)} are both inductive, graph-level benchmarks, with three regression tasks defined over each benchmark. For each task, we follow the implementations and optimized hyperparameters for BiGINE and GINE from~\citet{huang2025good}, the latter operating on the undirected version of the graphs.

\textbf{OCB}~\citep{dong2023cktgnn} consists of 8,026 operational amplifier-based (op-amp) analog circuits as directed graphs, and the task is to predict DC gain (Gain), bandwidth (BW) and phase margin (PM) for each circuit. We evaluate two Multi-$q$-MagLapPE configurations, each with ten $q$ values but with phase parameters of 0.01 and 0.05. Following \citet{huang2025good}, we use random 90/5/5\% splits for training, validation and testing over three-stage amplifiers only.\footnote{While \citet{huang2025good} claim that they report two-stage results, their codebase and results are consistent with three-stage amplifiers, while two-stage amplifiers are presented as a separate OOD generalization task.} 

\textbf{HLS}~\citep{wu2022high} represents 18,570 C/C++ programs as directed intermediate representation (IR) graphs for performance analysis on FPGA integrated circuits. The task is to predict digital signal processor (DSP) and look-up table (LUT) usage, and critical path timing slack (CP, permitted delay that does not affect overall program runtime) for each program. We evaluate two Multi-$q$-MagLapPE configurations, each with five $q$ values but with phase parameters of 0.01 and 0.1. We again follow \citet{huang2025good} and randomly select 1,000 graphs each for validation and testing, and train on the remaining graphs.

\paragraph{Experimental setup.} Hyperparameters for all \textsc{OCB} and \textsc{HLS} experiments are provided in Tables~\ref{tab:hparams-amp} and~\ref{tab:hparams-hls}, respectively. We observed several discrepancies between our results and those of \citet{huang2025good}. Primarily, despite using the suggested optimized model hyperparameters and $q$ values, our Multi-$q$-MagLapPE results were marginally worse than those reported in \citet{huang2025good} on BiGINE, and significantly worse for GINE. For GINE in particular, Multi-$q$-MagLapPE suffered from instabilities despite using five or ten $q$ values: some seeds converged to reasonable losses, while others fully diverged over the course of training, as indicated by high error values accompanied by extreme standard deviations. Finally, despite no additional tuning, we noted that our baseline No PE \& LapPE configs for BiGINE $\times$ LUT significantly outperformed the numbers reported in \citet{huang2025good}.

\FloatBarrier

\input{tables/hparams_amp}

\input{tables/hparams_hls}

\subsection{Further implementation details}
\label{app:add_details}
Our codebase is based on PyTorch~\citep{pytorch}, PyTorch Geometric~\citep{Fey_PyG_2_0_Scalable_2025}, Hydra~\citep{Yadan2019Hydra} (released under an MIT license) and PyTorch Lightning~\citep{Falcon_PyTorch_Lightning_2019} (Apache 2.0 license). 

The path complex construction was adapted from \url{https://github.com/SteveHuntsmanBAESystems/PerformantPathHomology} (CC BY-NC-SA 4.0 License). We also use the magnetic Laplacian and SPE implementations in \url{https://github.com/Graph-COM/Multi-q-Maglap}; we also extend their repository for our OCB and HLS evaluations. All experiments were conducted on NVIDIA L40S GPUs, with 8 CPUs and 48 GB RAM.

\FloatBarrier


\section{Computational cost experiments}
\label{app:pe_cost}

PathLapPE uses more memory than Multi-$q$-MagLapPE, which processes $q$ values sequentially at fixed memory cost. For large graphs, sparse eigensolvers~\citep{scipy} reduce this memory at a modest cost in time. \Cref{tab:pe-cost-horizontal} reports time and memory on our benchmarks; on the largest, \textsc{Roman-Empire} (22,662 nodes), sparse PathLapPE takes three seconds and 70 MB, making computational costs negligible compared to training.

\input{tables/pe_cost_table_horizontal}

\end{document}

%% file: math_commands.tex
\usepackage{amsmath,amsfonts,bm}

\def\eqref#1{equation~\ref{#1}}
\def\Eqref#1{Equation~\ref{#1}}

\def\1{\bm{1}}

\DeclareMathAlphabet{\mathsfit}{\encodingdefault}{\sfdefault}{m}{sl}
\SetMathAlphabet{\mathsfit}{bold}{\encodingdefault}{\sfdefault}{bx}{n}



%% file: figures/motifs.tex
\begin{wrapfigure}[11]{r}{0.45\textwidth}
    \centering
    \vspace{-0.5em}

    \resizebox{0.96\linewidth}{!}{%
    \begin{tikzpicture}[
        vertex/.style={
            circle,
            fill=black,
            inner sep=1.7pt
        },
        edge/.style={
            -{Latex[length=1.7mm]},
            line width=0.9pt
        },
        path/.style={
            edge,
            line width=1.2pt
        },
        altpath/.style={
            path,
            densely dashed
        },
        context/.style={
            edge,
            gray!55,
            line width=0.7pt
        },
        every node/.style={font=\small\itshape}
    ]

    \begin{scope}
        \node[vertex,label=below left:$a$]  (a1) at (0,0) {};
        \node[vertex,label=below right:$b$] (b1) at (1.35,0) {};
        \node[vertex,label=above:$c$]       (c1) at (0.675,1.1) {};

        \draw[path]    (a1) -- (b1);
        \draw[path]    (b1) -- (c1);
        \draw[context] (a1) -- (c1);

        \node[font=\small] at (0.675,-0.48) {(a)};
    \end{scope}

    \begin{scope}[xshift=2.25cm]
        \node[vertex,label=below left:$a$] (a2)  at (0,0) {};
        \node[vertex,label=below right:$b$] (b2)  at (1.25,0) {};
        \node[vertex,label=above right:$c$] (c2)  at (1.25,1.1) {};
        \node[vertex,label=above left:$d$] (bp2) at (0,1.1) {};

        \draw[path]    (a2)  -- (b2);
        \draw[path]    (b2)  -- (c2);
        \draw[altpath] (a2)  -- (bp2);
        \draw[altpath] (bp2) -- (c2);

        \node[font=\small] at (0.625,-0.48) {(b)};
    \end{scope}

    \begin{scope}[xshift=4.45cm]
        \node[vertex,label=below:$a$] (a3) at (0,0.3) {};
        \node[vertex,label=below:$b$] (b3) at (1.2,0.3) {};

        \draw[path,bend left=23] (a3) to (b3);
        \draw[path,bend left=23] (b3) to (a3);

        \node[font=\small] at (0.6,-0.48) {(c)};
    \end{scope}

    \end{tikzpicture}%
    }

    \caption{Basis elements of $\Omega_2$: invariant elementary
    $2$-paths in (a) and (c), and the difference of the solid and
    dashed $2$-paths in (b).}
    \label{fig:directed-motifs}

\end{wrapfigure}

%% file: figures/pipeline/pathlappe.tex
\begingroup
\def\pathlappeCheckLib#1{%
  \expandafter\ifx\csname tikz@library@#1@loaded\endcsname\relax
    \PackageError{pathlappe-pipeline}{TikZ library `#1' is not loaded}%
      {Add \string\usetikzlibrary{arrows.meta,positioning,matrix,calc} to your preamble.}%
  \fi}
\pathlappeCheckLib{arrows.meta}\pathlappeCheckLib{positioning}%
\pathlappeCheckLib{matrix}\pathlappeCheckLib{calc}%
\definecolor{bluebd}{HTML}{5AA9E6}  \definecolor{bluebg}{HTML}{F2F8FD} \definecolor{bluehd}{HTML}{DCEEFB}
\definecolor{greenbd}{HTML}{4CAF50} \definecolor{greenbg}{HTML}{F3FAF2}\definecolor{greenhd}{HTML}{DDF1DB}
\definecolor{purpbd}{HTML}{A58AD8}  \definecolor{purpbg}{HTML}{F8F4FD} \definecolor{purphd}{HTML}{E9DFF8}
\definecolor{orangebd}{HTML}{F0A050}\definecolor{orangebg}{HTML}{FFF7EF}\definecolor{orangehd}{HTML}{FDE7D2}
\definecolor{nodeA}{HTML}{8EC8F2} \definecolor{nodeB}{HTML}{F59A9A}
\definecolor{nodeC}{HTML}{A8DB8E} \definecolor{nodeD}{HTML}{F7E27A}
\definecolor{grayfill}{HTML}{EFEFEF}
\definecolor{bluetx}{HTML}{1F3F66}  \definecolor{greentx}{HTML}{1E5128}
\definecolor{purptx}{HTML}{3E2A6B}  \definecolor{orangetx}{HTML}{6B3A12}
\definecolor{edgetx}{HTML}{7A2418}

\tikzset{
  every node/.append style={font=\small},
  vtx/.style={circle,draw=black!70,thick,minimum size=17pt,inner sep=0pt,font=\bfseries\itshape},
  dot/.style={circle,draw=black!60,minimum size=7pt,inner sep=0pt},
  garrow/.style={-{Stealth[length=7pt]},thick},
  flow/.style={-{Stealth[length=8pt,width=7pt]},line width=1.4pt,black!80},
  sub/.style={draw=#1,rounded corners=5pt,fill=white,fill opacity=0.55,text opacity=1},
  mat/.style={matrix of math nodes,left delimiter={[},right delimiter={]},
              inner sep=1pt,nodes={minimum width=13pt,text height=6.5pt,text depth=2pt,inner sep=1pt,font=\small},
              column sep=3pt,row sep=2.5pt},
  lab/.style={font=\footnotesize\itshape,inner sep=1pt},
}

\newcommand{\panel}[9]{%
  \path[draw=#5,fill=#6,rounded corners=8pt,line width=0.9pt] (#1,#2) rectangle (#3,#4);
  \begin{scope}
    \clip[rounded corners=8pt] (#1,#2) rectangle (#3,#4);
    \fill[#7] (#1,#2) rectangle (#3,#2+55);
  \end{scope}
  \draw[#5,line width=0.6pt] (#1,#2+55) -- (#3,#2+55);
  \node[font=\large\bfseries,text=#9] at ({(#1+#3)/2},#2+28) {#8};
}

\begin{tikzpicture}[x=0.4pt,y=-0.4pt]

\panel{22}{17}{328}{427}{bluebd}{bluebg}{bluehd}{1.\ Directed graph}{bluetx}
\node at (175,115) {\large $G=(V,\vec{E})$};
\node[vtx,fill=nodeA] (a) at (174,178) {a};
\node[vtx,fill=nodeC] (c) at (91,253)  {c};
\node[vtx,fill=nodeB] (b) at (256,253) {b};
\node[vtx,fill=nodeD] (d) at (174,309) {d};
\draw[garrow] (c) -- (a);
\draw[garrow] (a) -- (b);
\draw[garrow] (c) -- (b);
\draw[garrow] (d) -- (b);
\node at (175,357) {$V=\{a,b,c,d\}$};
\node at (175,393) {$\vec{E}=\{ca,ab,cb,db\}$};

\panel{357}{17}{855}{427}{bluebd}{bluebg}{bluehd}{2.\ Path complex bases}{bluetx}
\draw[sub=bluebd] (368,85) rectangle (495,415);
\draw[sub=bluebd] (504,85) rectangle (706,415);
\draw[sub=bluebd] (715,85) rectangle (843,415);
\node[font=\large] at (431,106) {$\Omega_0$};
\node[font=\small] at (431,139) {(nodes)};
\node[dot,fill=nodeA] (n0a) at (423,195) {}; \node[lab,right=3pt of n0a] {$a$};
\node[dot,fill=nodeC] (n0c) at (404,243) {}; \node[lab,left=3pt of n0c]  {$c$};
\node[dot,fill=nodeB] (n0b) at (459,243) {}; \node[lab,right=3pt of n0b] {$b$};
\node[dot,fill=nodeD] (n0d) at (423,291) {}; \node[lab,right=3pt of n0d] {$d$};
\node at (431,347) {$B_0=$};
\node at (431,378) {$(a,b,c,d)$};
\node[font=\large] at (605,106) {$\Omega_1$};
\node[font=\small] at (605,139) {(edges)};
\newcommand{\miniedge}[6]{
  \draw[#5,thick] (#1-26,#2) -- (#1+26,#2);
  \node[dot,fill=#3] at (#1-26,#2) {};
  \node[dot,fill=#4] at (#1+26,#2) {};
  \node[lab] at (#1,#2+25) {#6};}
\node[font=\small] at (553,181) {$c\!\to\! a$}; \miniedge{553}{203}{nodeC}{nodeA}{greenbd}{$ca$}
\node[font=\small] at (655,181) {$a\!\to\! b$}; \miniedge{655}{203}{nodeA}{nodeB}{bluebd!70!black}{$ab$}
\node[font=\small] at (553,269) {$c\!\to\! b$}; \miniedge{553}{291}{nodeC}{nodeB}{red!70!black}{$cb$}
\node[font=\small] at (655,269) {$d\!\to\! b$}; \miniedge{655}{291}{nodeD}{nodeB}{red!60!black}{$db$}
\node at (605,352) {$B_1=$};
\node at (605,382) {$(ca,ab,cb,db)$};
\node[font=\large] at (779,106) {$\Omega_2$};
\node[font=\small] at (779,139) {(2-paths)};
\node[font=\small] at (779,198) {$c\!\to\! a\!\to\! b$};
\draw[green!50!black,thick] (740,222) -- (779,222);
\draw[bluebd!70!black,thick] (779,222) -- (818,222);
\draw[-{Stealth[length=5pt]},dashed,thick] (740,226) to[bend right=50] (815,230);
\node[dot,fill=nodeC] at (740,222) {};
\node[dot,fill=nodeA] at (779,222) {};
\node[dot,fill=nodeB] at (818,222) {};
\node[lab] at (779,278) {$cab$};
\node at (779,347) {$B_2=$};
\node at (779,381) {$(cab)$};

\panel{887}{17}{1342}{427}{greenbd}{greenbg}{greenhd}{3.\ Boundary matrices}{greentx}
\draw[sub=greenbd] (898,85) rectangle (1127,415);
\draw[sub=greenbd] (1137,85) rectangle (1332,415);
\node[font=\large] at (1013,110) {$\mathbf{B}_1$};
\node[font=\small] at (1013,141) {(edges $\to$ nodes)};
\coordinate (pB1) at (1025,290);
\begin{scope}[y=0.4pt]\matrix[mat] (B1) at (pB1) {
   1 & -1 &  0 &  0 \\
   0 &  1 &  1 &  1 \\
  -1 &  0 & -1 &  0 \\
   0 &  0 &  0 & -1 \\};\end{scope}
\foreach \c/\t in {1/ca,2/ab,3/cb,4/db} \node[lab,above=6pt] at (B1-1-\c.north) {$\t$};
\foreach \r/\t in {1/a,2/b,3/c,4/d} \node[lab,anchor=east] at ($(B1-\r-1.west)+(-11pt,0)$) {$\t$};
\node[font=\large] at (1235,110) {$\mathbf{B}_2$};
\node[font=\small] at (1235,141) {(2-paths $\to$ edges)};
\coordinate (pB2) at (1255,290);
\begin{scope}[y=0.4pt]\matrix[mat] (B2) at (pB2) { 1 \\ 1 \\ -1 \\ 0 \\};\end{scope}
\node[lab,above=6pt] at (B2-1-1.north) {$cab$};
\foreach \r/\t in {1/ca,2/ab,3/cb,4/db} \node[lab,anchor=east] at ($(B2-\r-1.west)+(-11pt,0)$) {$\t$};

\panel{1374}{17}{1828}{427}{purpbd}{purpbg}{purphd}{4.\ Path Laplacians}{purptx}
\draw[sub=purpbd] (1385,87) rectangle (1592,415);
\draw[sub=purpbd] (1602,87) rectangle (1820,415);
\node[font=\normalsize] at (1488,122) {$\mathbf{L}_0=\mathbf{B}_1\mathbf{B}_1^{\top}$};
\node[font=\small] at (1488,157) {(node Laplacian)};
\coordinate (pL0) at (1488,290);
\begin{scope}[y=0.4pt]\matrix[mat] (L0) at (pL0) {
   2 & -1 & -1 &  0 \\
  -1 &  3 & -1 & -1 \\
  -1 & -1 &  2 &  0 \\
   0 & -1 &  0 &  1 \\};\end{scope}
\node[font=\small] at (1711,122) {$\mathbf{L}_1=\mathbf{B}_2\mathbf{B}_2^{\top}+\mathbf{B}_1^{\top}\mathbf{B}_1$};
\node[font=\small] at (1711,157) {(edge Laplacian)};
\coordinate (pL1) at (1711,290);
\begin{scope}[y=0.4pt]\matrix[mat] (L1) at (pL1) {
  3 & 0 & 0 & 0 \\
  0 & 3 & 0 & 1 \\
  0 & 0 & 3 & 1 \\
  0 & 1 & 1 & 2 \\};\end{scope}

\draw[flow] (330,212) -- (355,212);
\draw[flow] (857,212) -- (885,212);
\draw[flow] (1344,212) -- (1372,212);
\draw[flow] (1601,429) -- (1601,455);  

\panel{1206}{457}{1828}{858}{bluebd}{bluebg}{bluehd}{5.\ Spectral decomposition}{bluetx}
\draw[sub=bluebd] (1216,523) rectangle (1512,846);
\draw[sub=bluebd] (1522,523) rectangle (1818,846);
\node[font=\normalsize\bfseries] at (1364,557) {Node eigenpairs (\textit{p}\,=\,0)};
\node[font=\large] at (1364,628) {$\mathbf{L}_0u_i^{(0)}=\lambda_i^{(0)}u_i^{(0)}$};
\node[font=\normalsize,align=center] at (1364,710) {Keep $k_{\mathrm{node}}$ smallest\\eigenpairs};
\node[font=\Large] at (1364,792) {$\mathbf{U}_0,\Lambda_0$};
\node[font=\normalsize\bfseries] at (1670,557) {Edge eigenpairs (\textit{p}\,=\,1)};
\node[font=\large] at (1670,628) {$\mathbf{L}_1u_i^{(1)}=\lambda_i^{(1)}u_i^{(1)}$};
\node[font=\normalsize,align=center] at (1670,710) {Keep $k_{\mathrm{edge}}$ smallest\\eigenpairs};
\node[font=\Large] at (1670,792) {$\mathbf{U}_1,\Lambda_1$};

\panel{578}{457}{1168}{858}{orangebd}{orangebg}{orangehd}{6.\ PathLapPE features}{orangetx}
\draw[draw=bluebd,fill=bluebg,rounded corners=6pt] (588,520) rectangle (1158,646);
\node[font=\normalsize\bfseries,text=bluetx,anchor=west] at (596,543) {Node SPE (\textit{p}\,=\,0)};
\node[font=\normalsize,anchor=west] at (596,584)
  {$\mathbf{K}^{(0)}_\ell=\sum_i \phi_\ell(\lambda^{(0)})_i\,u_i^{(0)}u_i^{(0)\top},\ \ \ell=1,\dots,m$};
\node[font=\normalsize,anchor=west] at (596,624)
  {$\boldsymbol{\to}\ z_v^{(0)}=\rho\big(\mathbf{K}^{(0)}_1,\dots,\mathbf{K}^{(0)}_m\big)_v$};
\draw[draw=orangebd,fill=orangebg,rounded corners=6pt] (588,656) rectangle (1158,848);
\node[font=\normalsize\bfseries,text=edgetx,anchor=west] at (596,679) {Edge SPE (\textit{p}\,=\,1)};
\node[font=\normalsize,anchor=west] at (596,720)
  {$\mathbf{K}^{(1)}_\ell=\sum_i \phi_\ell(\lambda^{(1)})_i\,u_i^{(1)}u_i^{(1)\top},\ \ \ell=1,\dots,m$};
\node[font=\normalsize,anchor=west] at (596,761)
  {$\boldsymbol{\to}\ z_e^{(1)}=\mathrm{MLP}\big(\big[\mathbf{K}^{(1)}_\ell(e,e),\,\sum_{f\in\mathcal{S}(e)}\mathbf{K}^{(1)}_\ell(e,f),$};
\node[font=\normalsize,anchor=west] at (596,797)
  {$\hphantom{\boldsymbol{\to}\ z_e^{(1)}=\mathrm{MLP}\big(\big[}\sum_{f\in\mathcal{P}(e)}\mathbf{K}^{(1)}_\ell(f,e)\big]_{\ell=1}^{m}\big)$};
\node[font=\small,text=black!65,anchor=west] at (596,830)
  {$\mathcal{S}(e)$, $\mathcal{P}(e)$: successor and predecessor edges of $e$};

\panel{22}{457}{540}{858}{purpbd}{purpbg}{purphd}{7.\ GNN and downstream task}{purptx}
\draw[fill=grayfill,draw=black!70,rounded corners=6pt] (111,562) rectangle (451,640);
\node[font=\large\bfseries] at (281,601) {GNN architecture};
\draw[-{Stealth[length=6pt]},thick] (281,642) -- (281,678);
\draw[fill=purphd!60,draw=black!70,rounded corners=6pt] (42,680) rectangle (520,834);
\node[font=\large\bfseries] at (281,732) {Graph/Node/Edge-level task};
\node[font=\normalsize] at (281,784) {e.g., node classification, graph regression};

\draw[flow] (1204,648) -- (1170,648);
\draw[flow] (576,648) -- (542,648);

\end{tikzpicture}
\endgroup

%% file: tables/webkb.tex
\begin{table}
  \caption{
  Node classification accuracy (mean test accuracy \(\%\)) on the
  \textsc{WebKB} datasets.
  The best result within each model family is in bold and the second best
  is underlined.
  }
  \label{tab:WebKB}
  \centering
  \resizebox{\linewidth}{!}{
  \begin{tabular}{llcccc}
    \toprule
    \textbf{Model}
    & \textbf{PE}
    & \textsc{Cornell}
    & \textsc{Wisconsin}
    & \textsc{Texas}
    & \textbf{Avg. Rank} $\downarrow$ \\
    \midrule

    \multirow{6}{*}{\rotatebox[origin=c]{90}{GAT}}
    & No PE
    & \(48.37 \pm 7.59\)
    & \(52.55 \pm 4.11\)
    & \(59.73 \pm 4.12\)
    & 5.17 \\

    & RWPE
    & \(54.05 \pm 7.96\)
    & \(53.72 \pm 6.55\)
    & \(60.00 \pm 4.38\)
    & 3.17 \\

    & LapPE
    & \(54.05 \pm 7.75\)
    & \(52.35 \pm 8.82\)
    & \(59.73 \pm 4.84\)
    & 4.67 \\

    & Multi-\(q\)-MagLapPE,
      \(q \in \{0.01,0.05,0.10\}\)
    & \(\underline{55.68 \pm 7.67}\)
    & \(53.14 \pm 7.01\)
    & \(\underline{60.27 \pm 4.24}\)
    & 2.67 \\

    & PathLapPE (ours)
    & \(53.78 \pm 6.17\)
    & \(\mathbf{56.28 \pm 5.31}\)
    & \(59.46 \pm 4.59\)
    & 4.00 \\

    & PathLapPE (ours, w/ EdgeSPE)
    & \(\mathbf{55.95 \pm 8.55}\)
    & \(\underline{55.88 \pm 7.69}\)
    & \(\mathbf{61.35 \pm 4.04}\)
    & \(\mathbf{1.33}\) \\

    \midrule

    \multirow{6}{*}{\rotatebox[origin=c]{90}{BiGINE}}
    & No PE
    & \(54.32 \pm 4.84\)
    & \(58.43 \pm 3.43\)
    & \(56.22 \pm 9.52\)
    & 5.83 \\

    & RWPE
    & \(55.41 \pm 8.66\)
    & \(58.43 \pm 5.37\)
    & \(61.35 \pm 6.25\)
    & 5.17 \\

    & LapPE
    & \(\mathbf{64.60 \pm 5.41}\)
    & \(76.86 \pm 9.33\)
    & \(64.60 \pm 9.57\)
    & 3.17 \\

    & Multi-\(q\)-MagLapPE,
      \(q \in \{0.01,0.05,0.10\}\)
    & \(62.97 \pm 10.74\)
    & \(\underline{79.02 \pm 6.34}\)
    & \(69.19 \pm 7.56\)
    & 3.00 \\

    & PathLapPE (ours)
    & \(\underline{64.05 \pm 7.21}\)
    & \(78.43 \pm 8.57\)
    & \(\mathbf{71.62 \pm 8.76}\)
    & 2.33 \\

    & PathLapPE (ours, w/ EdgeSPE)
    & \(\mathbf{64.60 \pm 7.90}\)
    & \(\mathbf{80.59 \pm 5.88}\)
    & \(\underline{69.73 \pm 10.57}\)
    & \(\mathbf{1.50}\) \\

    \midrule

    \multirow{6}{*}{\rotatebox[origin=c]{90}{BiGCNE}}
    & No PE
    & \(60.27 \pm 5.98\)
    & \(69.41 \pm 6.08\)
    & \(65.14 \pm 8.40\)
    & 5.33 \\

    & RWPE
    & \(58.92 \pm 6.60\)
    & \(65.88 \pm 6.80\)
    & \(68.11 \pm 7.07\)
    & 5.67 \\

    & LapPE
    & \(76.75 \pm 4.45\)
    & \(85.49 \pm 5.08\)
    & \(81.89 \pm 4.04\)
    & 3.83 \\

    & Multi-\(q\)-MagLapPE,
      \(q \in \{0.01,0.05,0.10\}\)
    & \(\underline{77.30 \pm 3.17}\)
    & \(85.49 \pm 4.45\)
    & \(\mathbf{83.24 \pm 7.07}\)
    & 2.33 \\

    & PathLapPE (ours)
    & \(77.03 \pm 5.87\)
    & \(\underline{86.67 \pm 5.13}\)
    & \(\mathbf{83.24 \pm 7.07}\)
    & 2.17 \\

    & PathLapPE (ours, w/ EdgeSPE)
    & \(\mathbf{78.65 \pm 4.50}\)
    & \(\mathbf{87.26 \pm 3.84}\)
    & \(\underline{82.16 \pm 4.80}\)
    & \(\mathbf{1.67}\) \\

    \bottomrule
  \end{tabular}
  }
\end{table}

%% file: tables/directed_roman_empire.tex
\begin{table}
    \centering
    \caption{Node classification results (mean test acc. \(\%\)) on \textsc{Roman-Empire}. Best result is \textbf{bolded}, second best is \underline{underlined}. Best $k$ (num. eigenpairs) reported, Table~\ref{tab:directed_roman_empire_complete} lists results for all $k$ values.}
    \resizebox{\linewidth}{!}{
    \begin{tabular}{lccccc}\toprule
         \textbf{PE}&  \textbf{BiGINE}&  \textbf{BiGINE-shared}&  \textbf{GAT}&  \textbf{BiGAT}& \textbf{BiGCNE}\\\midrule
         Base (No PE)  &  $89.48 \pm 0.40$&  $83.03 \pm 0.63$&  $83.01 \pm 0.35$&  $90.36 \pm 0.30$& $87.79 \pm 0.47$\\
         LaplacianPE   &  $90.97 \pm 0.42$&  $87.68 \pm 0.42$&  $84.19 \pm 0.48$&  $\mathbf{91.15 \pm 0.40}$& $89.58 \pm 0.57$\\
         Multi-$q$-MagLapPE, $q=0.05$&  $90.95 \pm 0.36$&  $87.66 \pm 0.24$&  $83.93 \pm 0.56$&  $91.13 \pm 0.44$& $89.63 \pm 0.36$\\
         Multi-$q$-MagLapPE, $\#q=3, q\in \{0.05, 0.1, 0.15\}$&  $90.95 \pm 0.33$&  $87.28 \pm 0.39$&  \underline{$84.32 \pm 0.39$}&  $91.10 \pm 0.45$& $89.82 \pm 0.35$\\\midrule
         PathLapPE (ours) &  \underline{$91.33 \pm 0.23$}&  $\mathbf{89.50 \pm 0.36}$&  $\mathbf{84.39 \pm 0.39}$&  $91.10 \pm 0.46$& $\mathbf{90.29 \pm 0.32}$\\
         PathLapPE (ours, w/ EdgeSPE) &  $\mathbf{91.37 \pm 0.32}$&  \underline{$89.17 \pm 0.56$}&  $83.90 \pm 0.42$&  \underline{$91.14 \pm 0.27$}& \underline{$90.06 \pm 0.45$}\\ 
         \bottomrule
    \end{tabular}
    }
    \label{tab:directed_roman_empire}
\end{table}

%% file: tables/amp_hsl.tex
\begin{table}[t]
\centering
\caption{Test set results for \textsc{Open Circuit Benchmark} and
\textsc{High-level Synthesis}. Mean $\pm$ std reported across five seeds. Best result is \textbf{bolded}, second best is \underline{underlined}. Multi-$q$-MagLapPE uses $\#q=10$ for OCB, and $\#q=5$ for HLS.}
\label{tab:amp_hsl}
\resizebox{\linewidth}{!}{
\begin{tabular}{@{}c@{\hspace{12pt}}lccccccc@{}}
\toprule
 & & \multicolumn{3}{c}{\textsc{OCB} (RMSE $\downarrow$)} & \multicolumn{3}{c}{\textsc{HLS} (MSE $\downarrow$)} & \textbf{Avg.} \\
\cmidrule(lr){3-5} \cmidrule(lr){6-8}
 & \textbf{PE} & \textsc{Gain} & \textsc{BW} & \textsc{PM} & \textsc{DSP} & \textsc{LUT} & \textsc{CP} & \textbf{Rank} $\downarrow$ \\
\midrule
\multirow{6}{*}{\rotatebox[origin=c]{90}{BiGINE}}
 & No PE     & \val{0.394}{0.007} & \val{4.497}{0.111} & \val{1.123}{0.011} & \bval{2.254}{0.068} & \val{1.793}{0.174} & \val{1.152}{0.077} & $4.08$ \\
 & LapPE     & \val{0.394}{0.014} & \val{4.223}{0.077} & \val{1.114}{0.022} & \val{2.445}{0.216} & \bval{1.767}{0.030} & \val{1.158}{0.132} & $3.75$ \\
 & Multi-$q$-MagLapPE, $\#q=\{5, 10\}, q=0.01$     & \val{0.393}{0.010} & \uval{4.156}{0.143} & \uval{1.112}{0.015} & \val{2.325}{0.103} & \val{1.783}{0.052} & \val{1.186}{0.107} & \underline{$3.17$} \\
 & Multi-$q$-MagLapPE, $\#q=\{5, 10\}, q \in \{0.05, 0.1\}$ & \val{0.388}{0.015} & \val{4.215}{0.128} & \val{1.121}{0.018} & \val{2.335}{0.137} & \val{1.856}{0.134} & \val{1.383}{0.348} & $4.17$ \\
 & PathLapPE (ours) & \uval{0.384}{0.014} & \val{4.236}{0.102} & \val{1.147}{0.005} & \val{2.625}{0.121} & \val{2.096}{0.083} & \bval{1.113}{0.093} & $4.33$ \\
 & PathLapPE (ours, w/ EdgeSPE) & \bval{0.367}{0.005} & \bval{4.036}{0.092} & \bval{1.102}{0.026} & \uval{2.322}{0.083} & \uval{1.771}{0.131} & \uval{1.118}{0.069} & $\mathbf{1.50}$ \\
\midrule
\multirow{6}{*}{\rotatebox[origin=c]{90}{GINE}}
 & No PE     & \val{0.415}{0.007} & \val{4.817}{0.110} & \val{1.126}{0.019} & \bval{2.662}{0.215} & \uval{2.009}{0.092} & \bval{0.654}{0.034} & \underline{$2.92$} \\
 & LapPE     & \val{0.411}{0.020} & \val{4.349}{0.076} & \val{1.127}{0.030} & \val{2.738}{0.279} & \val{2.031}{0.139} & \uval{0.680}{0.025} & $3.33$ \\
 & Multi-$q$-MagLapPE, $\#q=\{5, 10\}, q=0.01$     & \uval{0.410}{0.017} & \val{4.328}{0.159} & \bval{1.117}{0.015} & \val{2.885}{0.251} & \val{21.35}{37.64} & \val{1.639}{2.152} & $3.50$ \\
 & Multi-$q$-MagLapPE, $\#q=\{5, 10\}, q \in \{0.05, 0.1\}$ & \val{1.109}{0.226} & \val{44.98}{39.94} & \val{1.120}{0.016} & \val{2.910}{0.034} & \val{13.67}{18.62} & \val{2.795}{2.281} & $5.17$ \\
 & PathLapPE (ours) & \val{0.415}{0.021} & \uval{4.252}{0.146} & \val{1.149}{0.024} & \val{3.043}{0.209} & \val{2.353}{0.095} & \val{0.691}{0.047} & $4.25$ \\
 & PathLapPE (ours, w/ EdgeSPE) & \bval{0.379}{0.009} & \bval{4.226}{0.128} & \uval{1.119}{0.019} & \uval{2.717}{0.137} & \bval{1.946}{0.061} & \val{0.697}{0.046} & $\mathbf{1.83}$ \\
\bottomrule
\end{tabular}
}
\end{table}

%% file: tables/dataset_stats_node.tex
\begin{table}[h]
    \centering
\caption{Dataset statistics for transductive node-level tasks}
\label{tab:dataset_stats_node}
    \begin{tabular}{cccccc}\toprule
         \textbf{Dataset}&  \textsc{Cornell}&  \textsc{Wisconsin}&  \textsc{Texas}& \textsc{Roman-Empire}\\
         \midrule
         \# Nodes&  183&  251& 183  & 22,662 \\
         \# Edges&  295&  499& 325   & 44,363 \\
         \# Features&  1,703&  1,703& 1,703 & 300 \\
         \# Classes&  5&  5& 5 & 18 \\ 
         License& CC-BY 4.0& CC-BY 4.0& CC-BY 4.0& MIT \\
         \bottomrule
    \end{tabular}
\end{table}

%% file: tables/dataset_stats_graph.tex
\begin{table}[h]
    \centering
\caption{Dataset statistics for inductive graph-level tasks}
\label{tab:dataset_stats_graph}
\resizebox{\linewidth}{!}{
    \begin{tabular}{cccccc}\toprule
         \textbf{Dataset}&  \textsc{FlowGraph2}&  \textsc{FlowGraph3}&  \textsc{FlowGraph6}& \textsc{OCB}& \textsc{HLS}\\
         \midrule
         \# Graphs&  2,000&  3,000&  6,000&  8,026\textsuperscript{a}& 18,570\\
         Avg. \# Nodes&  114.46&  111.66&  111.40&  9.89& 94.70\\
         Avg. \# Edges&  150.30&  150.70&  149.93&  15.02& 122.12\\
         \# Node Features&  0&  0&  0&  3& 7\\
         \# Edge Features&  0&  0&  0&  0& 2\\
         Task&  Binary class.&  3-way class.&  6-way class.& Regression& Regression\\
         License& CC BY 4.0& CC BY 4.0& CC BY 4.0&  CC BY 4.0& CC BY 4.0\\
         \bottomrule
    \end{tabular}
}
\begin{minipage}{\textwidth}\footnotesize
\textsuperscript{a}\,The full dataset consists of 10,000 graphs, but we focus on three-stage amplifier graphs only -- two-stage amplifiers are presented as a separate OOD generalization task. While \citet{huang2025good} denote that they report two-stage results, their codebase and results are consistent with three-stage amplifiers instead.
\smallskip

\end{minipage}
\end{table}

%% file: tables/directed_roman_empire_complete.tex
\begin{table}[htbp]
\centering
\caption{Results across encoder architectures and positional encodings on \textsc{Roman-Empire}. Bold indicates the best $k$ for each variant. In the last column, the top three PEs for each encoder are denoted with \textcolor{YellowOrange}{\textbf{gold}}, \textcolor{CadetBlue}{\textbf{silver}} and \textcolor{RawSienna}{\textbf{bronze}}.}
\label{tab:directed_roman_empire_complete}
\setlength{\tabcolsep}{5pt}
\renewcommand{\arraystretch}{1.15}
\small
\resizebox{\linewidth}{!}{
\begin{tabular}{@{}llcccc!{\color{bestgray}\vrule width 0.8pt}c@{}}
\toprule
\textbf{Encoder} & \textbf{PE Variant} & $k{=}8$ & $k{=}16$ & $k{=}32$ & $k{=}64$ & \textbf{Best} \\
\midrule
 
\multirow{7}{*}{\rotatebox[origin=c]{90}{\parbox{2.2cm}{\centering\scriptsize BiGINE-shared}}}
 & Base (No PE)  & $83.03 \pm 0.63$ & $83.03 \pm 0.63$ & $83.03 \pm 0.63$ & $83.03 \pm 0.63$ & $83.03 \pm 0.63$ \\
 & LaplacianPE   & $\mathbf{87.68 \pm 0.42}$ & $87.44 \pm 0.23$ & $87.60 \pm 0.31$ & $87.53 \pm 0.42$ & $87.68 \pm 0.42$ \\
 & Multi-\(q\)-MagLapPE, $q=0.05$ & $\mathbf{87.66 \pm 0.24}$ & $87.47 \pm 0.36$ & $87.52 \pm 0.42$ & $87.64 \pm 0.40$ & $87.66 \pm 0.24$ \\
 & Multi-\(q\)-MagLapPE, $q \in \{0.05, 0.1, 0.15\}$ & $87.10 \pm 0.68$ & $\mathbf{87.28 \pm 0.39}$ & $87.17 \pm 0.52$ & $87.16 \pm 0.53$ & $87.28 \pm 0.39$ \\
 & PathLapPE (256 $\times$ 2) & $89.31 \pm 0.44$ & $\mathbf{89.43 \pm 0.42}$ & $89.36 \pm 0.41$ & $89.25 \pm 0.37$ & $\second{89.43 \pm 0.42}$ \\
 & PathLapPE       & $\mathbf{89.50 \pm 0.36}$ & $89.15 \pm 0.29$ & $89.12 \pm 0.41$ & $89.25 \pm 0.41$ & $\first{89.50 \pm 0.36}$ \\
 & PathLapPE (EdgeSPE) & $89.11 \pm 0.48$ & $\mathbf{89.17 \pm 0.56}$ & $89.01 \pm 0.57$ & $89.05 \pm 0.45$ & $\third{89.17 \pm 0.56}$ \\ 
\cmidrule(r){1-6}
 
\multirow{7}{*}{\rotatebox[origin=c]{90}{\parbox{2.2cm}{\centering\scriptsize BiGINE}}}
 & Base (No PE)  & $89.48 \pm 0.40$ & $89.48 \pm 0.40$ & $89.48 \pm 0.40$ & $89.48 \pm 0.40$ & $89.48 \pm 0.40$ \\
 & LaplacianPE   & $90.93 \pm 0.44$ & $\mathbf{90.97 \pm 0.42}$ & $90.84 \pm 0.38$ & $90.89 \pm 0.31$ & $90.97 \pm 0.42$ \\
 & Multi-\(q\)-MagLapPE, $q=0.05$ & $90.93 \pm 0.34$ & $90.90 \pm 0.30$ & $\mathbf{90.95 \pm 0.36}$ & $90.86 \pm 0.32$ & $90.95 \pm 0.36$ \\
 & Multi-\(q\)-MagLapPE, $q \in \{0.05, 0.1, 0.15\}$ & $90.86 \pm 0.23$ & $\mathbf{90.95 \pm 0.33}$ & $90.94 \pm 0.37$ & $90.91 \pm 0.30$ & $90.95 \pm 0.33$ \\
 & PathLapPE (256 $\times$ 2) & $91.28 \pm 0.30$ & $91.30 \pm 0.31$ & $\mathbf{91.38 \pm 0.29}$ & $91.31 \pm 0.34$ & $\first{91.38 \pm 0.29}$ \\
 & PathLapPE       & $91.28 \pm 0.30$ & $91.03 \pm 0.28$ & $91.19 \pm 0.32$ & $\mathbf{91.33 \pm 0.23}$ & $\third{91.33 \pm 0.23}$ \\
 & PathLapPE (EdgeSPE) & $91.28 \pm 0.38$ & $\mathbf{91.37 \pm 0.32}$ & $91.30 \pm 0.29$ & $91.37 \pm 0.26$ & $\second{91.37 \pm 0.32}$ \\ 
\cmidrule(r){1-6}
 
\multirow{7}{*}{\rotatebox[origin=c]{90}{\parbox{2.2cm}{\centering\scriptsize GAT}}}
 & Base (No PE)  & $83.01 \pm 0.35$ & $83.01 \pm 0.35$ & $83.01 \pm 0.35$ & $83.01 \pm 0.35$ & $83.01 \pm 0.35$ \\
 & LaplacianPE   & $83.94 \pm 0.82$ & $83.69 \pm 0.68$ & $\mathbf{84.19 \pm 0.48}$ & $84.06 \pm 0.47$ & $\third{84.19 \pm 0.48}$ \\
 & Multi-\(q\)-MagLapPE, $q=0.05$ & $83.76 \pm 0.41$ & $83.79 \pm 0.54$ & $83.75 \pm 0.45$ & $\mathbf{83.93 \pm 0.56}$ & $83.93 \pm 0.56$ \\
 & Multi-\(q\)-MagLapPE, $q \in \{0.05, 0.1, 0.15\}$ & $\mathbf{84.32 \pm 0.39}$ & $84.26 \pm 0.65$ & $84.03 \pm 0.42$ & $83.79 \pm 0.60$ & $\second{84.32 \pm 0.39}$ \\
 & PathLapPE (256 $\times$ 2) & $83.72 \pm 0.64$ & $\mathbf{84.17 \pm 0.31}$ & $83.87 \pm 0.61$ & $83.96 \pm 0.52$ & $84.17 \pm 0.31$ \\
 & PathLapPE       & $\mathbf{84.39 \pm 0.39}$ & $83.93 \pm 0.48$ & $84.20 \pm 0.45$ & $83.97 \pm 0.48$ & $\first{84.39 \pm 0.39}$ \\
 & PathLapPE (EdgeSPE) & $83.80 \pm 0.61$ & $\mathbf{83.90 \pm 0.42}$ & $83.69 \pm 0.76$ & $83.54 \pm 0.54$ & $83.90 \pm 0.42$ \\
\cmidrule(r){1-6}
 
\multirow{7}{*}{\rotatebox[origin=c]{90}{\parbox{2.2cm}{\centering\scriptsize BiGAT}}}
 & Base (No PE)  & $90.36 \pm 0.30$ & $90.36 \pm 0.30$ & $90.36 \pm 0.30$ & $90.36 \pm 0.30$ & $90.36 \pm 0.30$ \\
 & LaplacianPE   & $\mathbf{91.15 \pm 0.40}$ & $91.03 \pm 0.39$ & $90.97 \pm 0.59$ & $91.06 \pm 0.44$ & $\second{91.15 \pm 0.40}$ \\
 & Multi-\(q\)-MagLapPE $q=0.05$ & $91.12 \pm 0.37$ & $90.93 \pm 0.25$ & $90.74 \pm 0.60$ & $\mathbf{91.13 \pm 0.44}$ & $91.13 \pm 0.44$ \\
 & Multi-\(q\)-MagLapPE, $q \in \{0.05, 0.1, 0.15\}$ & $91.06 \pm 0.45$ & $90.95 \pm 0.44$ & $\mathbf{91.10 \pm 0.45}$ & $91.04 \pm 0.35$ & $91.10 \pm 0.45$ \\
 & PathLapPE (256 $\times$ 2) & $91.01 \pm 0.36$ & $90.93 \pm 0.42$ & $\mathbf{91.19 \pm 0.33}$ & $90.94 \pm 0.42$ & $\first{91.19 \pm 0.33}$ \\
 & PathLapPE       & $90.99 \pm 0.34$ & $90.91 \pm 0.37$ & $91.07 \pm 0.44$ & $\mathbf{91.10 \pm 0.46}$ & $91.10 \pm 0.46$ \\
 & PathLapPE (EdgeSPE) & $90.95 \pm 0.35$ & $91.05 \pm 0.36$ & $90.91 \pm 0.49$ & $91.14 \pm 0.27$ & $\third{91.14 \pm 0.27}$ \\ 
\cmidrule(r){1-6}
 
\multirow{7}{*}{\rotatebox[origin=c]{90}{\parbox{2.2cm}{\centering\scriptsize BiGCNE}}}
 & Base (No PE)  & $87.79 \pm 0.47$ & $87.79 \pm 0.47$ & $87.79 \pm 0.47$ & $87.79 \pm 0.47$ & $87.79 \pm 0.47$ \\
 & LaplacianPE   & $89.31 \pm 0.51$ & $\mathbf{89.58 \pm 0.57}$ & $89.40 \pm 0.41$ & $89.46 \pm 0.53$ & $89.58 \pm 0.57$ \\
 & Multi-\(q\)-MagLapPE, $q=0.05$ & $89.50 \pm 0.53$ & $\mathbf{89.63 \pm 0.36}$ & $89.51 \pm 0.39$ & $89.53 \pm 0.52$ & $89.63 \pm 0.36$ \\
 & Multi-\(q\)-MagLapPE, $q \in \{0.05, 0.1, 0.15\}$ & $\mathbf{89.82 \pm 0.35}$ & $89.66 \pm 0.35$ & $89.77 \pm 0.59$ & $89.76 \pm 0.45$ & $89.82 \pm 0.35$ \\
 & PathLapPE (256 $\times$ 2) & $90.14 \pm 0.31$ & $\mathbf{90.24 \pm 0.28}$ & $90.21 \pm 0.40$ & $90.20 \pm 0.48$ & $\second{90.24 \pm 0.28}$ \\
 & PathLapPE       & $90.22 \pm 0.34$ & $90.24 \pm 0.56$ & $\mathbf{90.29 \pm 0.32}$ & $90.23 \pm 0.32$ & $\first{90.29 \pm 0.32}$ \\
 & PathLapPE (EdgeSPE) & $89.99 \pm 0.44$ & $\mathbf{90.06 \pm 0.45}$ & $90.04 \pm 0.45$ & $89.99 \pm 0.48$ & $\third{90.06 \pm 0.45}$ \\
 
\bottomrule
\end{tabular}
}
\end{table}

\clearpage

%% file: tables/hparams_exps.tex
\begin{table}[htbp]
\centering
\small
\setlength{\tabcolsep}{5pt}
\caption{Key model and training hyperparameters for the \textsc{FlowGraph}, \textsc{WebKB} and \textsc{Roman-Empire} datasets. All \textsc{FlowGraph} and \textsc{WebKB} datasets share the same configs.}
\label{tab:hparams}
\resizebox{\linewidth}{!}{
\begin{tabular}{lccc}
\toprule
 & \textsc{FlowGraph} & \textsc{WebKB} & \textsc{Roman-Empire} \\
\midrule
\multicolumn{4}{l}{\textbf{\emph{Data}}}\\
Task & graph cls. & node cls. & node cls. \\
Graphs & 2,000 / 3,000 / 6,000 & 1 & 1 \\
Classes & 2 / 3 / 6 & 5 & 18 \\
Input feature dim. & 1\textsuperscript{a} & 1,703 & 300 \\
Batching & batch size 32 & full graph & full graph \\
Splits & 500 / 300 / 200 per class & 10 fixed splits & 10 fixed splits \\
\midrule
\multicolumn{4}{l}{\textbf{\emph{Optimization}}}\\
Optimizer & AdamW & Adam & AdamW \\
Learning rate & $8\cdot10^{-3}$ & $5\cdot10^{-3}$ & $3\cdot10^{-5}$ \\
Weight decay & $0$ & $5\cdot10^{-4}$ & $0$ \\
LR schedule & none & ReduceLROnPlateau\textsuperscript{b} & none \\
Max epochs & 200 & 3,000 & 1,000 \\ 
Early stopping (val/acc) & -- & 500 & 400 \\
\midrule
\multicolumn{4}{l}{\textbf{\emph{Backbone}}}\\
Layers (pre / MP / post) & 1 / 4 / 1 & 1 / 2 / 1 & 1 / 3 / 1 \\
Stage type & skipsum & -- & skipsum \\
Jumping knowledge & -- & max & -- \\
FFN after convolution & -- & -- & \checkmark \\
Hidden dim.\ $d$ & 32 & 128 & 512 \\
Normalization & BatchNorm & L2 Norm & LayerNorm \\
Dropout & 0.0 & \{0.0, 0.5\} & 0.2 \\
Activation & ReLU & ReLU & GELU \\
Graph pooling & mean & -- & -- \\
\midrule
\multicolumn{4}{l}{\textbf{\emph{Positional encodings}}}\\
LapPE $k$ & 64 & \{16, 32\} & \{8, 16, 32, 64\} \\
MagLapPE $k\,/\,q\,/\,\#q$ & 64 / 0.05 / \{1, 3, 5\} & \{16, 32\} / \{0.01, 0.05, 0.1\} / 1 & \{8, 16, 32, 64\} / \{0.05, 0.1, 0.15\} / \{1, 3\} \\
PathLapPE $k$ (node / edge) & 64 & \{16, 32\} / \{16, 32\} & \{8, 16, 32, 64\} \\
Encoder & SPE (+ EdgeSPE) & SPE (+ EdgeSPE) & SPE (+ EdgeSPE) \\
\bottomrule
\end{tabular}
}
\\[3pt]
\begin{minipage}{\textwidth}\footnotesize
\textsuperscript{a}\,FlowGraph has no node features; we attach a
constant $x_i = 1$.\\
\textsuperscript{b}\,$\gamma = 0.1$, patience 200, monitoring validation loss.\\

\end{minipage}
\end{table}

%% file: tables/hparams_amp.tex
\begin{table}[htbp]
\centering
\caption{Key model and training hyperparameters for the \textsc{Open Circuit
Benchmark} experiments. All backbones (both uni- and bidirectional GINE) share a single configuration per target.}
\label{tab:hparams-amp}
\small
\begin{tabular}{lccc}
\toprule
 & \textsc{Gain} & \textsc{BW} & \textsc{PM} \\
\midrule
\textit{\textbf{Data}} \\
Input node features   & \multicolumn{3}{c}{3} \\ 
Input edge features               & \multicolumn{3}{c}{--} \\
Graphs                      & 8,026 & 8,026 & 8,026 \\
Splits                      & $90/5/5$\% & $90/5/5$\% & $90/5/5$\% \\
Batch size                  & 128 & 64 & 64 \\
\midrule
\textit{\textbf{Optimization}} \\
Optimizer                   & \multicolumn{3}{c}{Adam ($\beta_1, \beta_2 = 0.9, 0.999$)} \\
Learning rate               & $6.7\cdot10^{-3}$ & $6.5\cdot10^{-3}$ & $2.1\cdot10^{-3}$ \\
Weight decay                & 0 & 0 & 0 \\
LR schedule                 & \multicolumn{3}{c}{StepLR (step 100, $\gamma = 0.95$)} \\
Loss                        & MSE & MSE & MSE \\
Max epochs                  & \multicolumn{3}{c}{300} \\
Model selection             & \multicolumn{3}{c}{best val.\ MSE checkpoint (no early stopping)} \\
\midrule
\textit{\textbf{Backbone}} \\
Layers (pre / MP / post)    & 1 / 4\textsuperscript{a} / 3 & 1 / 3\textsuperscript{a} / 2 & 1 / 4\textsuperscript{a} / 4 \\
Hidden dim.\ $d$            & 96 & 192 & 288 \\
Normalization               & \multicolumn{3}{c}{BatchNorm (inside each conv MLP)} \\
Dropout                     & 0.1 & 0.0 & 0.2 \\
Activation                  & ReLU & ReLU & ReLU \\
Pooling                     & sum & sum & sum \\
\bottomrule
\end{tabular}
\\[2pt]
\footnotesize
\textsuperscript{a}\,First over the device-grouping subgraph, then over the full circuit graph.
\end{table}

%% file: tables/hparams_hls.tex
\begin{table}[htbp]
\centering
\caption{Key model and training hyperparameters for the \textsc{High-Level
Synthesis} experiments. Unlike \textsc{OCB}, the tuned configuration depends on whether
the backbone is unidirectional (GINE) or bidirectional (BiGINE), so both are reported per target. A single value in a row indicates the parameter applies to all configurations.}
\label{tab:hparams-hls}
\small
\setlength{\tabcolsep}{4pt}
\begin{tabular}{lcccccc}
\toprule
 & \multicolumn{2}{c}{\textsc{DSP}} & \multicolumn{2}{c}{\textsc{LUT}} & \multicolumn{2}{c}{\textsc{CP}} \\
\cmidrule(lr){2-3}\cmidrule(lr){4-5}\cmidrule(lr){6-7}
 & Uni-GNN & Bi-GNN & Uni-GNN & Bi-GNN & Uni-GNN & Bi-GNN \\
\midrule
\textit{\textbf{Data}} \\
Categorical node features & \multicolumn{6}{c}{7} \\ 
Categorical edge features & \multicolumn{6}{c}{2} \\
Graphs (CDFG)             & \multicolumn{6}{c}{18,570} \\
Splits                    & \multicolumn{6}{c}{16,570 / 1,000 / 1,000} \\
Batch size                & 32 & 64 & 32 & 64 & 32 & 64 \\
\midrule
\textit{\textbf{Optimization}} \\
Optimizer      & \multicolumn{6}{c}{Adam ($\beta_1, \beta_2 = 0.9, 0.999$)} \\
Learning rate  & $5.92\cdot10^{-3}$ & $3.10\cdot10^{-3}$ & $1.50\cdot10^{-3}$ & $2.30\cdot10^{-3}$ & $5.92\cdot10^{-3}$ & $3.10\cdot10^{-3}$ \\
Weight decay   & \multicolumn{6}{c}{0} \\
LR schedule    & \multicolumn{6}{c}{StepLR (step 100, $\gamma = 0.95$)} \\
Loss           & \multicolumn{6}{c}{MSE} \\
Max epochs     & \multicolumn{6}{c}{400} \\
\midrule
\textit{\textbf{Backbone}} \\
Layers (pre / MP / post) & $1/2/2$ & $1/4/3$ & $1/4/2$ & $1/4/2$ & $1/2/2$ & $1/4/3$ \\
Hidden dim.\ $d$         & 56 & 84 & 84 & 56 & 56 & 84 \\
Normalization            & \multicolumn{6}{c}{BatchNorm (inside each conv MLP)} \\
Dropout                  & 0.1 & 0.2 & 0.3 & 0.1 & 0.1 & 0.2 \\
Activation               & \multicolumn{6}{c}{ReLU} \\
Pooling                  & sum & sum & sum & sum & max & max \\
\bottomrule
\end{tabular}
\end{table}

%% file: tables/pe_cost_table_horizontal.tex
\begin{table}[H]
\centering
\small
\setlength{\tabcolsep}{5pt}
\caption{Positional encoding precompute cost per graph. Each dataset is denoted with the optimal eigensolver strategy used for decomposition and the type of memory measured.}
\label{tab:pe-cost-horizontal}
\resizebox{\linewidth}{!}{
\begin{tabular}{lcccccc}
\toprule
 & \multicolumn{2}{c}{\textsc{Roman-Empire}} & \multicolumn{2}{c}{\textsc{FlowGraph}} & \multicolumn{2}{c}{\textsc{OCB}} \\
 & \multicolumn{2}{c}{\footnotesize shift-invert, heap mem.} & \multicolumn{2}{c}{\footnotesize CPU dense, heap mem.} & \multicolumn{2}{c}{\footnotesize GPU dense, GPU mem.} \\
\cmidrule(lr){2-3}\cmidrule(lr){4-5}\cmidrule(lr){6-7}
PE & Time (s) & Mem (MB) & Time (ms) & Mem (MB) & Time (ms) & Mem (MB) \\
\midrule
LapPE & 0.154\,$\pm$\,0.007 & 4.49 & 3.00\,$\pm$\,0.23 & 0.24 & 2.20\,$\pm$\,0.13 & 0.56 \\
Multi-$q$-MagLapPE, $\#q=1$ & 0.257\,$\pm$\,0.001 & 8.50 & 2.96\,$\pm$\,0.17 & 0.25 & 2.21\,$\pm$\,0.10 & 1.12 \\
Multi-$q$-MagLapPE, $\#q=3$ & 0.704\,$\pm$\,0.001 & 11.3 & 6.59\,$\pm$\,0.44 & 0.27 & 4.23\,$\pm$\,0.24 & 1.12 \\
Multi-$q$-MagLapPE, $\#q=5$ & 1.074\,$\pm$\,0.016 & 14.1 & 10.15\,$\pm$\,0.67 & 0.30 & 6.23\,$\pm$\,0.38 & 1.12 \\
Multi-$q$-MagLapPE, $\#q=10$ & 2.102\,$\pm$\,0.022 & 21.0 & 19.43\,$\pm$\,1.81 & 0.33 & 11.06\,$\pm$\,1.09 & 1.12 \\
PathLapPE (ours) & 3.053\,$\pm$\,0.015 & 70.5 & 22.46\,$\pm$\,1.15 & 0.62 & 14.31\,$\pm$\,0.83 & 0.60 \\
\bottomrule
\end{tabular}
}
\\
\end{table}